\pdfoutput=1
\documentclass{article}

\PassOptionsToPackage{numbers,compress}{natbib}

\usepackage[main,final]{neurips_2026}

\usepackage[utf8]{inputenc}
\usepackage[T1]{fontenc}
\usepackage{hyperref}
\usepackage{url}
\usepackage{booktabs}
\usepackage{multirow}
\usepackage{tabularx}
\usepackage{graphicx}
\usepackage{wrapfig}
\usepackage{subcaption}
\usepackage{amsmath}
\usepackage{amsfonts}
\usepackage{nicefrac}
\usepackage{microtype}
\usepackage{xcolor}
\usepackage{colortbl}
\usepackage{caption}
\definecolor{lightblue}{HTML}{E1F4FC}
\usepackage{amsthm}
\usepackage{enumitem}

\usepackage{titletoc}

\theoremstyle{definition}
\newtheorem{proposition}{Proposition}
\theoremstyle{plain}
\newtheorem{remark}{Remark}

\title{Diffusion Masked Pretraining for\\Dynamic Point Cloud}

\author{%
  Zhuoyue Zhang$^{1,2}$\thanks{Equal contribution.} \quad
  Yiding Sun$^{1}$\footnotemark[1] \quad
  Chaowei Fang$^{1,3}$ \quad
  Haozhe Cheng$^{1}$ \\
  \bfseries Jian Liu$^{4}$ \quad
  Ajmal Saeed Mian$^{5}$ \quad
  Jihua Zhu$^{1,2}$\thanks{Corresponding author.} \\
  \normalfont
  $^{1}$School of Software Engineering, Xi'an Jiaotong University, P.R. China \\
  $^{2}$State Key Laboratory of Human-Machine Hybrid Augmented Intelligence, \\
  Xi'an Jiaotong University, Xi'an, China \\
  $^{3}$State Key Lab. of Novel Software Technology, Nanjing University, P.R. China \\
  $^{4}$School of Artificial Intelligence and Robotics, Hunan University, China \\
  $^{5}$University of Western Australia, Perth, Australia \\
  \texttt{zhangzy4016@stu.xjtu.edu.cn} \quad
  \texttt{zhujh@xjtu.edu.cn} \\
}

\begin{document}

\maketitle

\begin{abstract}
Dynamic point cloud pretraining is still dominated by masked reconstruction objectives. However, these objectives inherit two key limitations. Existing methods inject ground-truth tube centers as decoder positional embeddings, causing spatio-temporal positional leakage. Moreover, they supervise inter-frame motion with deterministic proxy targets that systematically discard distributional structure by collapsing multimodal trajectory uncertainty into conditional means. To address these limitations, we propose \textbf{Di}ffusion \textbf{M}asked \textbf{P}retraining (\textbf{DiMP}), a unified self-supervised framework for dynamic point clouds. DiMP introduces diffusion modeling into both positional inference and motion learning. It first applies forward diffusion noise only to masked tube centers, then predicts clean centers from visible spatio-temporal context. This removes positional leakage while preserving visible coordinates as clean temporal anchors. DiMP also reformulates point-wise inter-frame displacement supervision as a DDPM noise-prediction objective conditioned on decoded representations. This design drives the encoder to target the full conditional distribution of plausible motions under a variational surrogate, rather than collapsing to a single deterministic estimate. Extensive experiments demonstrate that DiMP consistently improves downstream accuracy over the backbone alone, with absolute gains of 11.21\% on offline action segmentation and 13.65\% under causally constrained online inference. Code is available at \url{https://github.com/InitalZ/DiMP}.
\end{abstract}

\section{Introduction}

Self-supervised pretraining with masked autoencoders (MAE)~\cite{he2022masked,pang2022masked,yu2022point,sun2025hyperpoint,li2025pointdico} has proven effective for static 3D point cloud representation learning by masking local geometric regions and supervising coordinate reconstruction. Building on this, diffusion-based pretraining methods~\cite{zheng2024pointdif,li2024diffpmae,xiaopoint} integrate DDPMs~\cite{ho2020denoising} with masked autoencoding, and Point-MaDi~\cite{xiaopoint} further identifies that injecting ground-truth patch centers as decoder positional embeddings leaks spatial priors~\cite{sameni2023representation,wang2023droppos,caron2024location}, addressing this via diffusion noise over patch centers.

\begin{figure}
  \centering
  \includegraphics[width=\linewidth]{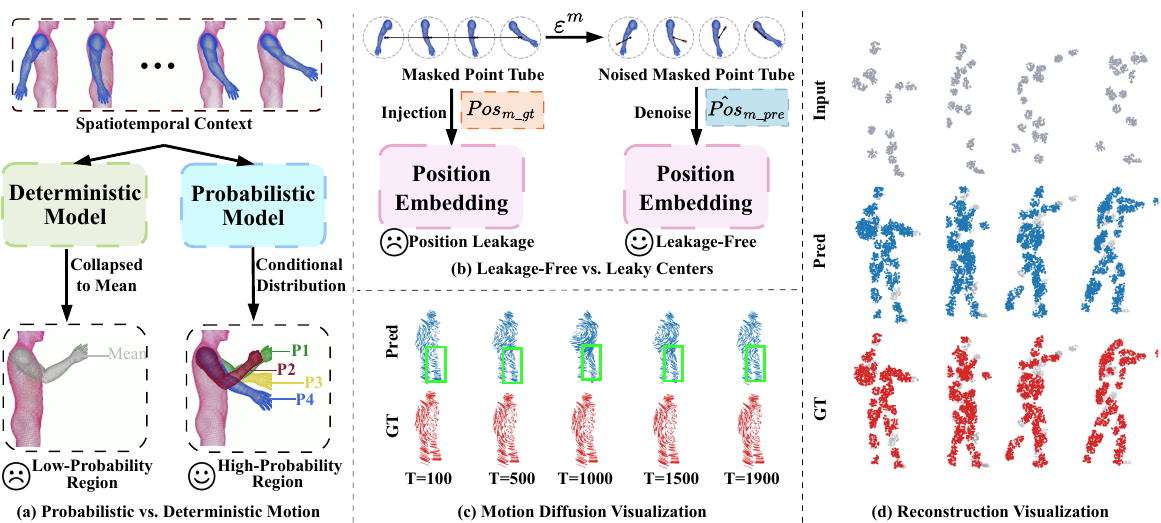}
  \caption{Motivation of DiMP.
(a) Deterministic models collapse multimodal motion distributions to a single conditional mean, which lies in a low-probability region when trajectories are multimodal. DiMP targets the full conditional distribution under a variational surrogate.
(b) Directly injecting ground-truth masked tube center coordinates as decoder
positional embeddings causes positional leakage, whereas DiMP applies forward
diffusion noise to masked centers and replaces ground-truth positions with
denoised predictions, yielding a leakage-free positional embedding.
(c) Qualitative comparisons of denoised inter-frame displacement predictions
against the ground truth show that DiMP recovers motion patterns faithfully
across representative motion diffusion timesteps $t_i^M$. Detailed analysis of why the model
performs well at small and large $t$ is provided in the Appendix~\ref{app:timestep-analysis}.
(d) Qualitative reconstruction results show that DiMP recovers the overall
human pose and preserves local geometric structures with high fidelity,
yielding reconstructions that closely match the ground truth.
}
\label{fig:motivation}
\end{figure}

Despite this progress, existing methods remain confined to \emph{static} point clouds.
Research has shifted toward dynamic point clouds~\cite{sun2026align}.
Such dynamic sequences couple spatial layout within each frame with temporal deformation across frames instead of treating scans as unrelated static clouds.
Dynamic point cloud understanding is critical to autonomous driving, embodied intelligence, and action recognition. Masked reconstruction methods such as MaST-Pre~\cite{shen2023masked} and M2PSC~\cite{han2024masked} extend the spatio-temporal tube masking paradigm to dynamic sequences. However, they still rely on deterministic proxy targets to supervise inter-frame motion. Accordingly, motion supervision is reduced to fitting conditional expectations. This formulation cannot target the full conditional distribution of motion trajectories. The same spatio-temporal context can correspond to multiple plausible continuations, but mean regression collapses these modes into a single deterministic solution, systematically discarding the distributional structure that is critical for discriminating action classes~\cite{gupta2018socialgan,salzmann2020trajectronplusplus}. Additionally, these methods inherit positional leakage from the static domain by injecting ground-truth spatio-temporal tube centers as decoder positional embeddings.

In this paper, we propose \textbf{Di}ffusion \textbf{M}asked \textbf{P}retraining (\textbf{DiMP}), the first framework to harness the full distributional modeling capacity of DDPMs for dynamic point cloud pretraining. DiMP addresses two intertwined challenges. \textbf{(1) Spatio-temporal positional leakage:} Unlike the static setting where noise can be applied uniformly to all patch centers, visible tube coordinates in videos serve a dual role as geometric anchors for the encoder and as temporal conditioning signals that link consecutive frames into a coherent motion sequence. Corrupting them would undermine motion continuity (Table~\ref{tab:ablation-noise}), necessitating perturbation exclusively on masked tube centers. \textbf{(2) Distributional motion modeling:} By reformulating inter-frame displacement supervision as a DDPM noise prediction task conditioned on decoded representations, DiMP drives the encoder to internalize the full multimodal distribution of motion trajectories rather than collapsing to deterministic mean estimates. Because different action classes can share nearly identical mean trajectories yet differ sharply in their trajectory variance and shape, capturing this distributional structure is decisive for fine-grained action discrimination.
Crucially, these two challenges are not independent, positional leakage is a \emph{prerequisite barrier} to distributional motion modeling. Leaked ground-truth tube centers collapse decoder cross-attention into local geometric retrieval, severing the informational pathway through which the motion diffusion branch drives the encoder to internalize motion distributions. Resolving positional leakage is therefore necessary before diffusion-based motion uncertainty can be genuinely exploited (see Appendix~\ref{app:leakage-prerequisite}).

Concretely, DiMP is built on a masked autoencoding backbone and operates in two stages. In the \emph{spatio-temporal center modeling} stage, we apply forward diffusion noise only to masked tube centers. The diffusion model is conditioned on visible features and temporal positional encodings to predict the masked center distribution. This design removes positional leakage while preserving clean temporal context. Based on the predicted centers, a masked decoder further fuses visible encoder features with predicted center embeddings to form complete spatio-temporal representations. These representations are supervised by per-frame Chamfer distance to reconstruct the full dynamic point cloud. In the \emph{motion-aware modeling} stage, we reformulate point-wise inter-frame displacement supervision as a standard DDPM noise prediction task conditioned on the decoded representation. This enables DiMP to genuinely exploit the distributional capacity of DDPMs in dynamic 3D pretraining.

Our main contributions are as follows:
\begin{itemize}
  \item We provide a formal argument that deterministic motion supervision
  systematically discards distributional structure critical to multimodal
  trajectory uncertainty, and further
  identify positional leakage as a prerequisite barrier to distributional motion
  modeling.

  \item We propose \textbf{DiMP}, the first diffusion-based pretraining framework
  for dynamic point clouds, which reformulates inter-frame displacement supervision
  as a DDPM noise-prediction task to target the full conditional distribution of
  motion trajectories under a variational surrogate.

  \item We introduce a center diffusion strategy that confines noise injection
  exclusively to masked tube centers, resolving positional leakage without
  compromising temporal inference.

  \item Extensive experiments across action segmentation, semantic segmentation,
  and action recognition demonstrate that DiMP consistently outperforms prior
  methods on multiple benchmarks, with particularly pronounced gains under
  causally constrained online inference.
\end{itemize}

\section{Related work}

\begin{figure}[t]
  \centering
  \includegraphics[width=\linewidth]{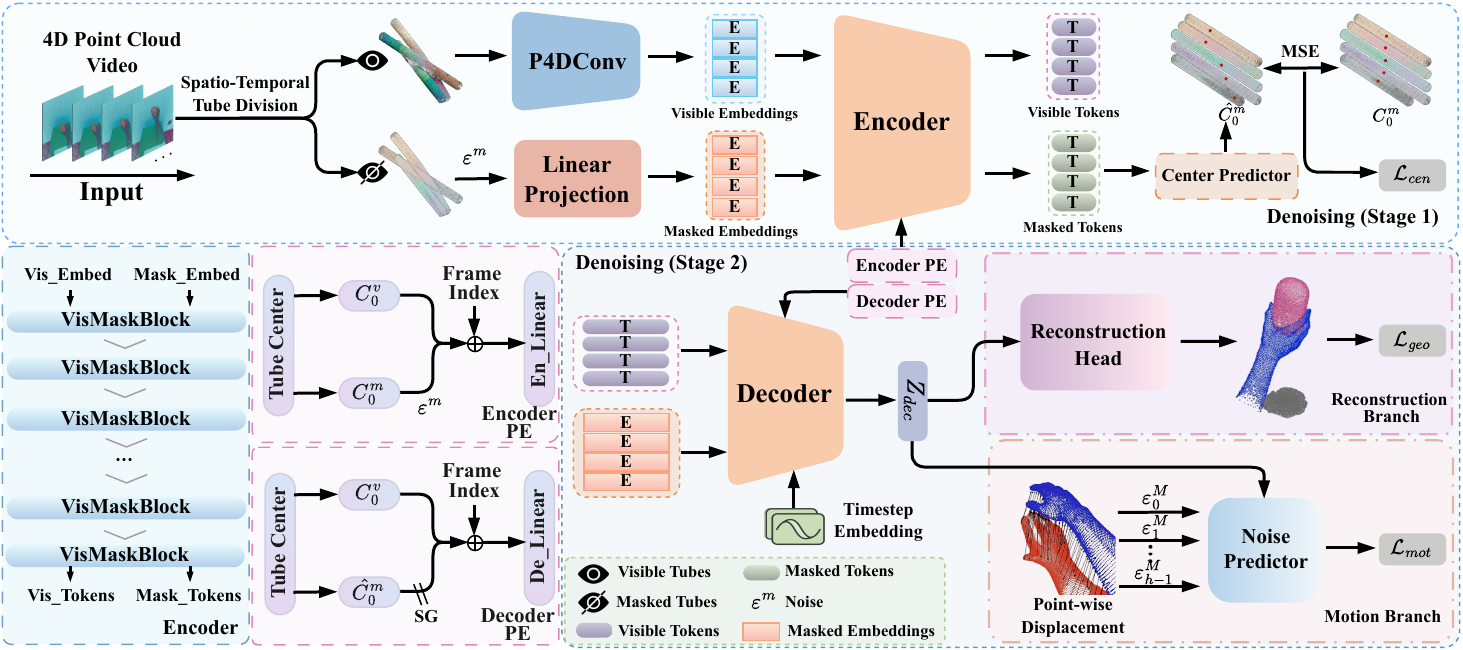}
  \caption{The pipeline of DiMP. Spatio-temporal tubes are extracted from a dynamic point cloud and randomly masked. P4DConv produces visible embeddings $\mathbf{e}_i$, while a linear projection maps corrupted masked point coordinates $\mathbf{P}_{t_c}^m$ to masked embeddings $\mathbf{e}_{t_c}^m$. In Stage~1, masked centers and point coordinates are independently perturbed by $\boldsymbol{\epsilon}_c$ and $\boldsymbol{\varepsilon}_c$ at the shared timestep $t_c$, while visible centers $\mathbf{C}_0^v$ remain clean. Encoder positional embeddings use $\mathbf{C}_0^v$, corrupted masked centers $\mathbf{C}_{t_c}^m$, and frame indices $t_f$. Stacked \texttt{VisMaskBlock} layers produce visible tokens $\mathbf{Z}^v$ and masked tokens $\mathbf{Z}^m$. The Center Predictor $g_m$ maps $\mathbf{Z}^m$ to clean-center predictions $\hat{\mathbf{C}}_0^m$, supervised against $\mathbf{C}_0^m$ by $\mathcal{L}_{\mathrm{cen}}$. In Stage~2, the decoder combines $\mathbf{Z}^v$ and $\mathbf{e}_{t_c}^m$ with positional embeddings from $\mathbf{C}_0^v$ and $\mathrm{sg}(\hat{\mathbf{C}}_0^m)$, frame indices $t_f$, and the timestep embedding $\tau(t_c)$ to produce $\mathbf{Z}_{\mathrm{dec}}$. Its masked-tube features support geometric reconstruction under $\mathcal{L}_{\mathrm{geo}}$, while the full output conditions the motion noise predictor $h_m$. At independently sampled stratified timesteps $t_i^M$, $i=0,\ldots,h-1$, the motion branch corrupts $\mathbf{M}_0$ with noise $\boldsymbol{\epsilon}_i^M$ and minimizes $\mathcal{L}_{\mathrm{mot}}$, covering fine-grained local recovery at small timesteps and coarse motion structure at large timesteps.}
  \label{fig:pipeline}
\end{figure}

\textbf{Self-supervised Learning on Static Point Clouds.}
Self-supervised learning for static point clouds has evolved along two principal paradigms, contrastive learning and reconstruction learning. Contrastive methods such as PointContrast~\cite{xie2020pointcontrast} and CrossPoint~\cite{afham2022crosspoint} learn invariant representations via point-level or cross-modal objectives. On the reconstruction learning side, Point-BERT~\cite{yu2022point} and Point-MAE~\cite{pang2022masked} established the masked autoencoding framework for point clouds, subsequently extended by hierarchical and cross-modal variants~\cite{zhang2022pointm2ae,zhang2023learning,zhang2026pointcot,wang2026pointrft,qi2023contrast}. More recently, diffusion-based reconstruction methods further advanced this line. PointDif~\cite{zheng2024pointdif} guides iterative denoising from noisy geometry via aggregated backbone features, DiffPMAE~\cite{li2024diffpmae} conditions a diffusion decoder on MAE encoder latents to reconstruct masked regions, and Point-MaDi~\cite{xiaopoint} applies diffusion modeling to patch centers alongside conditional patch reconstruction. Nevertheless, these methods remain tailored to static geometry and do not address the two challenges that arise in the dynamic setting: positional leakage in decoder conditioning~\cite{sameni2023representation,wang2023droppos,caron2024location} and the absence of distributional motion modeling. DiMP addresses both limitations by introducing masked-only center diffusion for positional inference and diffusion-based motion supervision for dynamic trajectories.

\textbf{Self-supervised Learning on Dynamic Point Cloud.}
Early approaches to dynamic point cloud representation learning rely on contrastive objectives, including temporal order prediction~\cite{wang2021self} and reconstruction-contrastive combinations~\cite{sheng2023contrastive,shen2023pointcmp,sheng2023point}. More relevant to our work, generative methods extend masked autoencoding to the spatio-temporal domain. MaST-Pre~\cite{shen2023masked} supervises point-tube reconstruction alongside temporal cardinality difference prediction, and related works~\cite{zuo2025uni4d,han2024masked} further develop this paradigm. However, these methods still suffer from positional leakage because they directly inject ground-truth spatio-temporal tube center coordinates as decoder positional embeddings. At the same time, they rely on coarse-grained deterministic proxy targets for motion supervision. This reduces learning to fitting conditional expectations and collapses multiple plausible motion modes into a single deterministic solution. Accordingly, the distributional structure critical for discriminating action classes is discarded~\cite{gupta2018socialgan,salzmann2020trajectronplusplus}. Motivated by these limitations, DiMP is the first pretraining framework for dynamic point clouds that jointly resolves positional leakage and distributional motion modeling within a unified diffusion-based formulation.
\section{Method}

We present \textbf{Di}ffusion \textbf{M}asked \textbf{P}retraining (\textbf{DiMP}),
a unified self-supervised pretraining framework for dynamic point clouds that
incorporates two complementary diffusion processes into a masked autoencoding paradigm.
Figure~\ref{fig:pipeline} provides an overview.
Starting from a dynamic point cloud, DiMP partitions it into spatio-temporal tubes
and randomly masks a large fraction of them.
In Stage~1 (\S\ref{sec:stage1}), a shared diffusion timestep $t_c$ controls
noise injection into both the masked tube \emph{centers} and
the masked tube \emph{point coordinates}.
The dual-stream encoder then predicts clean masked centers from this jointly corrupted context.
In Stage~2 (\S\ref{sec:stage2}), an asymmetric decoder combines visible encoder features
with the stop-gradient predicted centers, conditioned on $t_c$, to produce
$\mathbf{Z}_{\mathrm{dec}}$, which simultaneously feeds a geometric reconstruction head
and a motion-diffusion head.
The motion branch employs a \emph{separate} diffusion process over inter-frame
displacements, with independently sampled stratified timesteps.
All three objectives are unified via weighted summation (\S\ref{sec:total-loss}).

\subsection{Stage 1: token encoding and spatio-temporal center denoising}
\label{sec:stage1}

\paragraph{Spatio-temporal tube tokenization.}
A dynamic point cloud is partitioned into $K$ spatio-temporal tubes.
For each visible tube $i \in \mathcal{V}$, the $l \times n$ points within
its spatial neighborhood are tokenized by a P4DConv operator~\cite{fan2021point},
yielding a $d$-dimensional embedding $\mathbf{e}_i \in \mathbb{R}^d$.
The raw tube centers $\mathbf{C} = \{\hat{\mathbf{p}}_i\}_{i=1}^K$ are
globally mean-centered as $\mathbf{C}_0 = \mathbf{C} - \bar{\mathbf{C}}$
with $\bar{\mathbf{C}} = K^{-1}\sum_i\hat{\mathbf{p}}_i$,
suppressing rigid-body translation while preserving inter-frame motion signals.
After masking, $\mathbf{C}_0$ splits into visible $\mathbf{C}_0^v\!\in\!\mathbb{R}^{K_v\times 3}$
and masked $\mathbf{C}_0^m\!\in\!\mathbb{R}^{K_m\times 3}$.
Critically, $\mathbf{C}_0^v$ is kept \emph{noise-free} throughout both stages:
corrupting visible centers would destabilize the motion-continuity prior
underpinning temporal inference (confirmed empirically in Table~\ref{tab:ablation-noise}).

\begin{wrapfigure}{l}{0.5\linewidth}
      \vspace{-0.5\baselineskip}
      \centering
      \includegraphics[width=\linewidth]{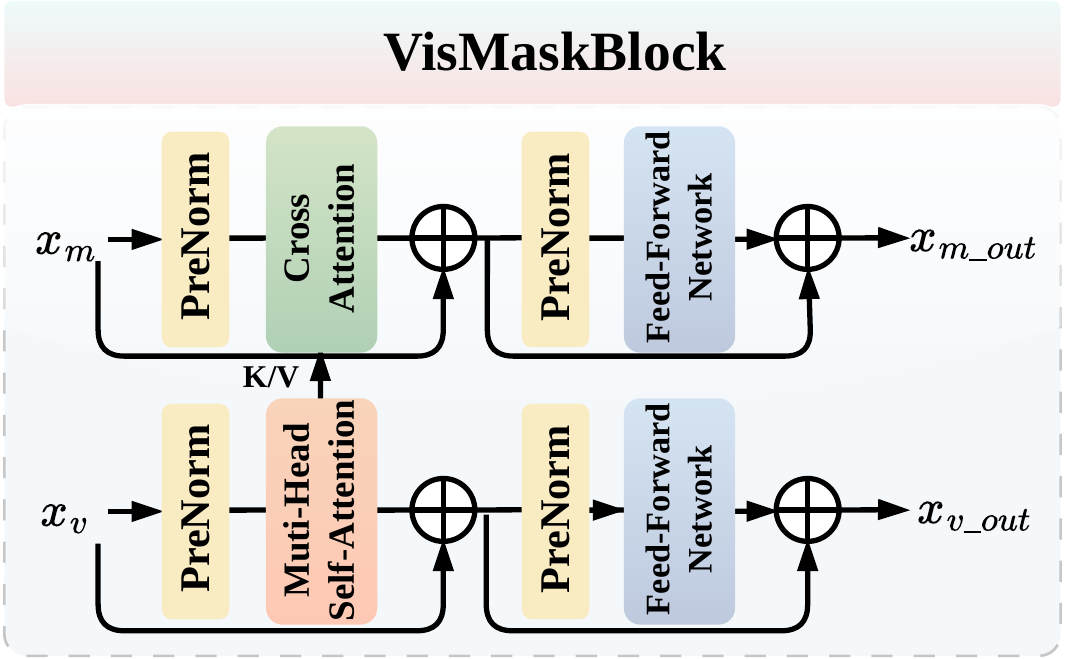}
      \caption{Structure of a \texttt{VisMaskBlock}. At layer $\ell$, visible tokens
      $\mathbf{Z}_{\ell}^v$ are updated by self-attention to yield $\mathbf{Z}_{\ell+1}^v$,
      while masked tokens $\mathbf{Z}_{\ell}^m$ query the visible stream through
      cross-attention to yield $\mathbf{Z}_{\ell+1}^m$.
      This asymmetric design preserves clean visible context and enables leakage-free
      inference for masked tubes.}
      \label{fig:vismaskblock}
\end{wrapfigure}

\paragraph{Shared diffusion schedule and corrupted inputs.}
A single center-diffusion timestep $t_c \sim \mathcal{U}\{1,T\}$ governs all
Stage~1 corruption.
Two independent noise samples are drawn at this shared timestep and applied
to different modalities of the \emph{same} masked tubes:
$\mathbf{C}^m_{t_c}=\sqrt{\bar{\alpha}_{t_c}}\,\mathbf{C}_0^m+\sqrt{1-\bar{\alpha}_{t_c}}\,\boldsymbol{\epsilon}_c$
with $\boldsymbol{\epsilon}_c \sim \mathcal{N}(\mathbf{0},\mathbf{I})$, and
$\mathbf{P}^m_{t_c}=\sqrt{\bar{\alpha}_{t_c}}\,\mathbf{P}^m_0+\sqrt{1-\bar{\alpha}_{t_c}}\,\boldsymbol{\varepsilon}_c$
with $\boldsymbol{\varepsilon}_c \sim \mathcal{N}(\mathbf{0},\mathbf{I})$,
where $\{\bar{\alpha}_t\}$ is the DDPM cosine schedule shared by both,
and $\mathbf{P}^m_0 \in \mathbb{R}^{n \times 3}$ denotes the raw point coordinates
of a masked tube.
Here $\boldsymbol{\epsilon}_c$ perturbs the masked tube \emph{centers} used in
positional embeddings, and $\boldsymbol{\varepsilon}_c$ perturbs the raw masked
\emph{point coordinates} used in content embeddings.
The two samples are drawn independently but follow the exact same noise
schedule and timestep $t_c$, so both modalities are corrupted at a
consistent level.
The perturbed coordinates are then projected to a masked content embedding
$\mathbf{e}^m_{t_c} = \mathrm{LinearProj}(\mathbf{P}^m_{t_c}) \in \mathbb{R}^d$.

\paragraph{Dual-stream encoder.}
The encoder $f_e$ processes visible and masked tokens through a shared stack of
\texttt{VisMaskBlock} layers (Figure~\ref{fig:vismaskblock}), where each block
applies $\mathbf{Z}^v_{\ell+1} = \mathrm{SelfAttn}(\mathbf{Z}^v_\ell)$ over visible
tokens and $\mathbf{Z}^m_{\ell+1} = \mathrm{CrossAttn}(\mathbf{Z}^m_\ell, \mathbf{Z}^v_\ell)$
from masked tokens to the visible context.
Each token is formed by adding its content embedding to
$\phi_{\mathrm{enc}}(\mathbf{c},\, t_f)$, an MLP-based encoder positional embedding
that maps tube center $\mathbf{c}$ concatenated with frame index $t_f$.
The full encoder forward pass produces:
\begin{align}
  \mathbf{Z}^v &= f_e^v\!\left(
    \bigl\{\mathbf{e}_i + \phi_{\mathrm{enc}}(\mathbf{C}_0^v,\, t_f)\bigr\}_{i \in \mathcal{V}}
  \right) \in \mathbb{R}^{K_v \times d}, \label{eq:encoder-vis} \\
  \mathbf{Z}^m &= f_e^m\!\left(
    \bigl\{\mathbf{e}^m_{t_c} + \phi_{\mathrm{enc}}(\mathbf{C}^m_{t_c},\, t_f)\bigr\}_{i \in \mathcal{M}},\;
    \mathbf{Z}^v
  \right) \in \mathbb{R}^{K_m \times d}, \label{eq:encoder-mask}
\end{align}
where $f_e^v$ applies self-attention over visible tokens and $f_e^m$ applies
cross-attention from masked tokens to the visible context.
Since cross-attention keys and values are drawn exclusively from $\mathbf{Z}^v$,
each masked hidden state $\mathbf{Z}^m_j$ ($j \in \mathcal{M}$) constitutes a
structured approximation to the denoising posterior
$q(\mathbf{C}_{0,j}^m \mid \mathbf{C}^m_{t_c},\mathbf{Z}^v)$.

\paragraph{Center prediction and Stage~1 loss.}
Center Predictor $g_m$ maps $\mathbf{Z}^m$ to predicted
clean centers $\hat{\mathbf{C}}_0^m = g_m(\mathbf{Z}^m)$, supervised by MSE:
\begin{equation}
  \mathcal{L}_{\mathrm{cen}}
  = \frac{1}{K_m}
    \sum_{j=1}^{K_m}
    \bigl\|\hat{\mathbf{C}}_{0,j}^m - \mathbf{C}_{0,j}^m\bigr\|_2^2.
  \label{eq:loss-cen}
\end{equation}
By supervising only the masked branch and varying corruption via $t_c$,
the model acquires robust positional representations under spatio-temporal uncertainty.
The parameters of $g_m$ receive gradients \emph{exclusively} from
$\mathcal{L}_{\mathrm{cen}}$, as enforced by the stop-gradient mechanism
described next.

\subsection{Stage 2: geometric reconstruction and motion-aware diffusion}
\label{sec:stage2}

\paragraph{Decoder with predicted-center and coordinate-timestep conditioning.}
With leakage-free center predictions $\hat{\mathbf{C}}_0^m$ from Stage~1 in hand,
an asymmetric decoder $f_d$ combines encoded visible tokens, noisy masked content
embeddings, decoder positional embeddings from clean visible and predicted
masked centers, and a timestep embedding of $t_c$:
\begin{equation}
  \mathbf{Z}_{\mathrm{dec}}
  = f_d\!\left(
      \mathbf{Z}^v,\;
      \mathbf{e}^m_{t_c},\;
      \phi_{\mathrm{dec}}(\mathbf{C}^v_0,\, t_f),\;
      \phi_{\mathrm{dec}}\!\bigl(\mathrm{sg}(\hat{\mathbf{C}}_0^m),\, t_f\bigr),\;
      \tau(t_c)
    \right)
    \in \mathbb{R}^{K \times d},
  \label{eq:decoder}
\end{equation}
where three design choices are worth noting.
The decoder positional embedding $\phi_{\mathrm{dec}}$ uses clean $\mathbf{C}_0^v$ for
visible tubes and the stop-gradient prediction $\mathrm{sg}(\hat{\mathbf{C}}_0^m)$ for
masked tubes, ensuring that $g_m$ is updated only by $\mathcal{L}_{\mathrm{cen}}$ and not
by the reconstruction or motion objectives (ablated in Table~\ref{tab:ablation-sg}).
The timestep token $\tau(t_c)$ is injected into $f_d$ to condition the decoder on the
geometry corruption level.

\paragraph{Geometric reconstruction.}
The masked-tube slice of $\mathbf{Z}_{\mathrm{dec}}$ is passed through a linear
reconstruction head to yield predicted point coordinates
$\hat{\mathbf{P}}^{\mathrm{pre}}$, supervised by the per-frame $\ell_2$
Chamfer distance averaged over all masked tubes:
\begin{equation}
  \mathcal{L}_{\mathrm{geo}}
  = \frac{1}{K_m} \sum_{i \in \mathcal{M}}
    \frac{1}{l} \sum_{s=1}^{l}
    \mathrm{CD}\!\left(
      \hat{\mathbf{P}}^{\mathrm{pre}}_{i,s},\;
      \mathbf{P}^{\mathrm{gt}}_{i,s}
    \right),
  \label{eq:loss-geo}
\end{equation}
where $\mathrm{CD}(\mathbf{A},\mathbf{B})
= \tfrac{1}{|\mathbf{A}|}\sum_{a\in\mathbf{A}}\min_{b\in\mathbf{B}}\|a{-}b\|_2^2
+ \tfrac{1}{|\mathbf{B}|}\sum_{b\in\mathbf{B}}\min_{a\in\mathbf{A}}\|b{-}a\|_2^2$.
The full $\mathbf{Z}_{\mathrm{dec}}$ then conditions the motion-aware diffusion module.

\paragraph{Inter-frame displacement as motion target.}
We construct point-wise inter-frame displacement vectors as the motion
supervision signal. For each point $\mathbf{p}_i^t$ in frame $t$, its
displacement is $\Delta\mathbf{p}_i^t = \mathbf{p}_i^{t+1} - \mathbf{p}_i^{t}$
for $i = 1, \ldots, N$ and $t = 1, \ldots, L-1$.
Point correspondences are established offline using the pre-trained
correspondence model CorrNet3D~\cite{zeng2021corrnet3d}, yielding the full motion matrix
$\mathbf{M}_0 \in \mathbb{R}^{N \times (L-1) \times 3}$ where
$\mathbf{M}_{0,i,t} = \Delta\mathbf{p}_i^t$.
CorrNet3D serves solely as an offline labeler whose parameters are frozen throughout
DiMP training with no gradient flowing through it, and it is discarded after pretraining
alongside the decoder and diffusion heads.
The encoder never accesses CorrNet3D features and only receives the scalar displacement
values stored as static supervision labels on disk.

\paragraph{Motion diffusion.}
The motion branch employs a diffusion process independent from the Stage~1 schedule.
To guarantee uniform coverage of the full diffusion hierarchy, $[1,T]$ is
partitioned into $h$ near-equal intervals with $d = \lfloor T/h \rfloor$:
\begin{equation}
  Q_i =
  \begin{cases}
    [d{\cdot}i+1,\; d{\cdot}(i+1)], & 0 \le i < h-1,\\
    [d(h-1)+1,\; T],                 & i = h-1,
  \end{cases}
\end{equation}
so that $\bigcup_{i=0}^{h-1}Q_i = [1,T]$.
For each interval $Q_i$, a timestep $t^M_i \sim \mathcal{U}(Q_i)$ and a noise
sample $\boldsymbol{\epsilon}^M_i \sim \mathcal{N}(\mathbf{0},\mathbf{I})$
are drawn independently, and the corresponding forward marginal
$\mathbf{M}_{t^M_i} = \sqrt{\bar{\alpha}_{t^M_i}}\,\mathbf{M}_0
+ \sqrt{1-\bar{\alpha}_{t^M_i}}\,\boldsymbol{\epsilon}^M_i$
is evaluated at all $h$ levels simultaneously.

\paragraph{Noise prediction head and loss.}
A lightweight PCNet-based MLP with a reset-gate conditioning
mechanism~\cite{zheng2024pointdif} serves as the noise prediction head $h_m$.
It predicts the injected motion noise at interval $i$ as
$\hat{\boldsymbol{\epsilon}}^M_i = h_m\!\left(\mathbf{M}_{t^M_i},\; t^M_i,\; \mathbf{Z}_{\mathrm{dec}}\right)$,
conditioned on the full decoder output $\mathbf{Z}_{\mathrm{dec}}$.
The motion diffusion loss averages the DDPM objective across all $h$ intervals:
\begin{equation}
  \mathcal{L}_{\mathrm{mot}}
  = \frac{1}{h}\sum_{i=0}^{h-1}
    \mathbb{E}_{\substack{t^M_i \sim Q_i,\\
                          \mathbf{M}_0,\,\boldsymbol{\epsilon}^M_i}}\!
    \left[\bigl\|
      \boldsymbol{\epsilon}^M_i -
      h_m\!\left(
        \sqrt{\bar{\alpha}_{t^M_i}}\mathbf{M}_0
        {+}\sqrt{1{-}\bar{\alpha}_{t^M_i}}\boldsymbol{\epsilon}^M_i,\;
        t^M_i,\;
        \mathbf{Z}_{\mathrm{dec}}
      \right)
    \bigr\|^2\right],
  \label{eq:loss-mot}
\end{equation}
\begin{wrapfigure}{r}{0.5\linewidth}
  \vspace{-0.5\baselineskip}
  \centering
  \includegraphics[width=\linewidth]{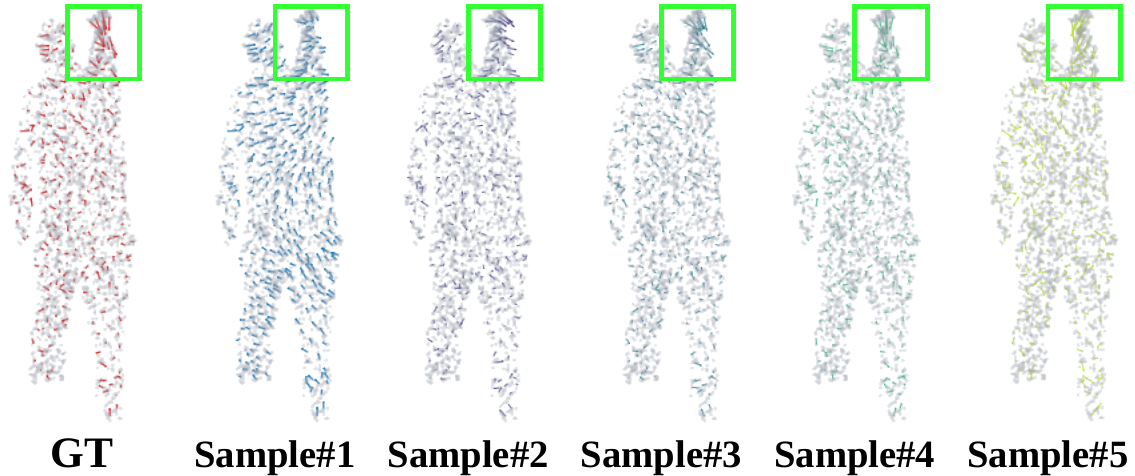}
  \caption{Five samples of clean inter-frame displacements from $p_\theta(\mathbf{M}_0\mid\mathbf{Z}_{\mathrm{dec}})$ with
  ground truth (red) for the same input. Sample diversity confirms multimodal
  trajectory modeling rather than mean collapse.}
  \label{fig:diversity-main}
\end{wrapfigure}
where the expectation is over motion sequences $\mathbf{M}_0$, motion noise
$\boldsymbol{\epsilon}^M_i$, and interval-stratified timesteps $t^M_i$, with
$\mathbf{Z}_{\mathrm{dec}}$ treated as a fixed conditioning signal.
Crucially, $\mathcal{L}_{\mathrm{mot}}$ need to propagate gradients through the decoder
all the way back to the encoder.
We partially mitigate the gradient-attenuation effect inherent to
decoder-conditioned objectives (discussed in Appendix~\ref{app:encoder-gradient})
by applying $\mathcal{L}_{\mathrm{mot}}$ simultaneously across $h$ stratified noise
levels within each training step, multiplying and diversifying the encoder gradient
signal from the motion branch.
Small-$t^M_i$ losses drive fine-grained local trajectory recovery, and large-$t^M_i$
losses compel coarse semantic motion understanding.
We set $h = 4$ by default (sensitivity analyzed in Table~\ref{tab:ablation-intervals}).

\paragraph{Theoretical motivation.}
The following result (Appendix~\ref{app:theory}) motivates DiMP's motion branch.

\begin{proposition}[Bayes-error obstruction of mean regression, informal]
\label{prop:info-loss-main}
If two action classes $a_1, a_2$ share the same class-conditional mean trajectory
$\mathbb{E}[\mathbf{M}\mid A{=}a_1]=\mathbb{E}[\mathbf{M}\mid A{=}a_2]$ but differ
in their conditional distributions, then a representation trained purely by
mean-regression motion supervision cannot separate $a_1$ from $a_2$ above chance.
A representation that targets the full distribution $p(\mathbf{M}\mid\mathbf{Z}_{\mathrm{dec}})$
retains more mutual information with the action label $A$.
\end{proposition}

\noindent Whenever two actions differ in how their trajectories vary rather than only in their
means, deterministic regression discards the distinguishing signal.
DiMP's variational surrogate avoids this by targeting the full conditional distribution.
Figure~\ref{fig:diversity-main} visualizes five independent samples of clean inter-frame
displacements drawn from $p_\theta(\mathbf{M}_0\mid\mathbf{Z}_{\mathrm{dec}})$ for the same input context, confirming that the model captures
multimodal trajectories rather than collapsing to a single estimate.

\subsection{Overall training objective}
\label{sec:total-loss}

The three losses are unified via weighted summation:
\begin{equation}
  \mathcal{L}
  = \mathcal{L}_{\mathrm{geo}}
  + \gamma_{\mathrm{cen}}\,\mathcal{L}_{\mathrm{cen}}
  + \lambda_{\mathrm{mot}}\,\mathcal{L}_{\mathrm{mot}},
  \label{eq:total-loss}
\end{equation}
where $\gamma_{\mathrm{cen}}$ weights the center diffusion loss and
$\lambda_{\mathrm{mot}}$ weights the motion diffusion loss.
Geometric reconstruction ($\mathcal{L}_{\mathrm{geo}}$) provides dense structural
supervision at unit weight, center diffusion ($\mathcal{L}_{\mathrm{cen}}$) acts
as a positional regularizer, and motion diffusion
($\mathcal{L}_{\mathrm{mot}}$) constitutes the primary novel contribution.
We set $\gamma_{\mathrm{cen}} = 0.1$ and $\lambda_{\mathrm{mot}} = 1.0$ by default
(sensitivity in Tables~\ref{tab:ablation-weight-cen} and~\ref{tab:ablation-weight-mot}).
All components are optimized jointly.
After pretraining, the decoder and both diffusion heads are discarded.
The encoder $f_e$ alone is retained for downstream fine-tuning.
\begin{table*}[!t]
  \caption{4D action segmentation results on the HOI4D dataset. All methods use 150 frames as input.
  Metrics: framewise accuracy (Acc), segmental edit distance (Edit), and F1 scores at overlap
  thresholds of 10\%, 25\%, and 50\%. Numbers in parentheses for our method denote absolute
  improvements over the backbone. $^\dagger$ denotes results reproduced from the public source code. Same as follows. Multi-seed statistics for DiMP and M2PSC are provided in
  Table~\ref{tab:multiseed}.}
  \label{tab:action-seg}
  \centering
  \scriptsize
  \setlength{\tabcolsep}{8pt}
  \begin{tabularx}{\textwidth}{Xlccccc}
    \toprule
    Method & Venue & Acc & Edit & F1@10 & F1@25 & F1@50 \\
    \midrule
    PPTr~\cite{wen2022point}                              & ECCV22        & 77.4 & 80.1 & 81.7 & 78.5 & 69.5 \\
    \hspace{0.8em}+ STRL~\cite{huang2021spatio}           & ICCV21        & 78.4 & 79.1 & 81.8 & 78.6 & 69.7 \\
    \hspace{0.8em}+ VideoMAE~\cite{tong2022videomae}      & NeurIPS22     & 78.6 & 80.2 & 81.9 & 78.7 & 69.9 \\
    \hspace{0.8em}+ C2P~\cite{zhang2023complete}          & CVPR23        & 81.1 & 84.0 & 85.4 & 82.5 & 74.1 \\
    \rowcolor{lightblue}
    \hspace{0.8em}+ DiMP (Ours) & \textbf{This Paper}
      & \textbf{83.97}\,\scalebox{0.55}{\textcolor{magenta}{$(+6.57)$}}
      & \textbf{86.15}\,\scalebox{0.55}{\textcolor{magenta}{$(+6.05)$}}
      & \textbf{89.13}\,\scalebox{0.55}{\textcolor{magenta}{$(+7.43)$}}
      & \textbf{86.64}\,\scalebox{0.55}{\textcolor{magenta}{$(+8.14)$}}
      & \textbf{81.43}\,\scalebox{0.55}{\textcolor{magenta}{$(+11.93)$}} \\
    \midrule
    P4Transformer~\cite{fan2021point}                     & CVPR21        & 71.2 & 73.1 & 73.8 & 69.2 & 58.2 \\
    \hspace{0.8em}+ C2P~\cite{zhang2023complete}          & CVPR23        & 73.5 & 76.8 & 77.2 & 72.9 & 62.4 \\
    \hspace{0.8em}+ MaST-Pre~\cite{shen2023masked}$^\dagger$        & ICCV23          & 74.1 & 75.4 & 76.6 & 73.4 & 63.8 \\
    \hspace{0.8em}+ M2PSC~\cite{han2024masked}$^\dagger$           & ECCV24          & 75.9 & 77.1 & 78.4 & 74.1 & 65.9 \\
    \rowcolor{lightblue}
    \hspace{0.8em}+ DiMP (Ours) & \textbf{This Paper}
      & \textbf{82.41}\,\scalebox{0.55}{\textcolor{magenta}{$(+11.21)$}}
      & \textbf{84.90}\,\scalebox{0.55}{\textcolor{magenta}{$(+11.80)$}}
      & \textbf{87.26}\,\scalebox{0.55}{\textcolor{magenta}{$(+13.46)$}}
      & \textbf{85.92}\,\scalebox{0.55}{\textcolor{magenta}{$(+16.72)$}}
      & \textbf{80.77}\,\scalebox{0.55}{\textcolor{magenta}{$(+22.57)$}} \\
    \bottomrule
  \end{tabularx}
\end{table*}

\section{Experiments}

\subsection{Downstream experiments}
\label{sec:downstream}

\paragraph{4D Action Segmentation on HOI4D.} As reported in Table~\ref{tab:action-seg}, DiMP achieves consistent improvements over all prior pre-training approaches on HOI4D, with gains especially pronounced in the Edit and F1 metrics, reflecting that a probabilistic motion objective promotes stronger temporal coherence than deterministic proxy targets. DiMP also yields substantially larger gains on the weaker P4Transformer backbone than on PPTr, while the final performance of both converges after pretraining, suggesting that P4Transformer's limited receptive field leaves greater capacity for DiMP's motion supervision to fill.

\paragraph{4D Semantic Segmentation on HOI4D.}
To verify that DiMP yields representations effective for fine-grained point-level understanding,
we conduct further experiments on the HOI4D semantic segmentation benchmark.
Mean Intersection over Union (mIoU, \%) is used as the evaluation metric.
As shown in Table~\ref{tab:sem-seg}, DiMP outperforms all prior pre-training methods.
This result shows that the probabilistic motion diffusion objective learns richer
geometric representations, since modeling temporal displacement distributions encourages
the encoder to capture fine-grained boundary geometry that is informative for both motion
understanding and static per-point labeling.
These representations transfer effectively to fine-grained per-point recognition tasks.

\begin{table*}[!t]
  \centering
  \begin{minipage}[t]{0.46\textwidth}
    \centering
    \captionof{table}{4D semantic segmentation results on HOI4D.}
    \label{tab:sem-seg}
    \footnotesize
    \setlength{\tabcolsep}{3pt}
    \begin{tabularx}{\linewidth}{>{\raggedright\arraybackslash}X l c}
      \toprule
      Method & Venue & mIoU (\%) \\
      \midrule
      PPTr~\cite{wen2022point}                          & ECCV22    & 41.0 \\
      \hspace{0.8em}+ STRL~\cite{huang2021spatio}       & ICCV21    & 41.2 \\
      \hspace{0.8em}+ VideoMAE~\cite{tong2022videomae}  & NeurIPS22 & 41.3 \\
      \hspace{0.8em}+ C2P~\cite{zhang2023complete}      & CVPR23    & 42.3 \\
      \rowcolor{lightblue}
      \hspace{0.8em}+ DiMP (Ours) & \textbf{This Paper} & \textbf{48.2}\,{\scalebox{0.55}{\textcolor{magenta}{$(+7.2)$}}} \\
      \midrule
      P4Transformer~\cite{fan2021point}                 & CVPR21    & 40.1 \\
      \hspace{0.8em}+ MaST-Pre~\cite{shen2023masked}$^\dagger$ & ICCV23 & 40.3 \\
      \hspace{0.8em}+ C2P~\cite{zhang2023complete}      & CVPR23    & 41.4 \\
      \hspace{0.8em}+ M2PSC~\cite{han2024masked}$^\dagger$ & ECCV24  & 42.3 \\
      \rowcolor{lightblue}
      \hspace{0.8em}+ DiMP (Ours) & \textbf{This Paper} & \textbf{47.6}\,{\scalebox{0.55}{\textcolor{magenta}{$(+7.5)$}}} \\
      \bottomrule
    \end{tabularx}
  \end{minipage}\hfill
  \begin{minipage}[t]{0.50\textwidth}
    \centering
    \captionof{table}{3D action recognition accuracy (\%) on MSRAction-3D.}
    \label{tab:msr}
    \footnotesize
    \setlength{\tabcolsep}{3pt}
    \begin{tabularx}{\linewidth}{>{\raggedright\arraybackslash}X l c}
      \toprule
      Method & Venue & Accuracy (\%) \\
      \midrule
      PSTNet~\cite{fan2021pstnet}                           & ICLR21  & 91.20 \\
      \hspace{0.8em}+ PointCMP~\cite{shen2023pointcmp}     & CVPR23  & 93.27 \\
      \midrule
      PST-Transformer~\cite{fan2022point}                   & TPAMI22 & 93.73 \\
      \hspace{0.8em}+ M2PSC~\cite{han2024masked}            & ECCV24  & 94.84 \\
      \hspace{0.8em}+ MaST-Pre~\cite{shen2023masked}        & ICCV23  & 94.08 \\
      \rowcolor{lightblue}
      \hspace{0.8em}+ DiMP (Ours) & \textbf{This Paper} & \textbf{95.51}\,{\scalebox{0.55}{\textcolor{magenta}{$(+1.78)$}}} \\
      \midrule
      P4Transformer~\cite{fan2021point}                     & CVPR21  & 90.94 \\
      \hspace{0.8em}+ M2PSC~\cite{han2024masked}            & ECCV24  & 93.03 \\
      \rowcolor{lightblue}
      \hspace{0.8em}+ DiMP (Ours) & \textbf{This Paper} & \textbf{93.97}\,{\scalebox{0.55}{\textcolor{magenta}{$(+3.03)$}}} \\
      \bottomrule
    \end{tabularx}
  \end{minipage}
\end{table*}

\begin{table*}[!t]
  \caption{Online action segmentation results on HOI4D under the causally constrained online
  inference setting. All methods operate on sequences of 150 frames. Runtime is measured in
  Clips/s, higher values indicate better real-time feasibility.}
  \label{tab:online}
  \centering
  \scriptsize
  \setlength{\tabcolsep}{5pt}
  \begin{tabularx}{\textwidth}{Xlcccccc}
    \toprule
    Method & Venue & Acc & Edit & F1@10 & F1@25 & F1@50 & Clips/s \\
    \midrule
    PointNet++~\cite{qi2017pointnet}         & CVPR17  & 50.8 & 40.2 & 42.6 & 35.3 & 23.5 & 35.8 \\
    P4Transformer~\cite{fan2021point}        & CVPR21  & 66.7 & 62.0 & 65.3 & 59.8 & 46.3 & 26.6 \\
    P4Mamba~\cite{liu2026pointnet4d}         & WACV26  & 71.0 & 78.5 & 77.8 & 73.4 & 61.5 & 33.4 \\
    PPTr~\cite{wen2022point}                 & ECCV22  & 69.7 & 64.5 & 69.4 & 64.8 & 52.9 & -- \\
    NSM4D~\cite{dong2023nsm4d}               & arXiv23 & 67.8 & 63.2 & 68.0 & 63.7 & 51.9 & -- \\
    NSM4D (600 frames)~\cite{dong2023nsm4d}  & arXiv23 & 71.3 & 68.0 & 72.1 & 68.1 & 56.5 & -- \\
    PointNet4D~\cite{liu2026pointnet4d}      & WACV26  & 70.3 & 72.3 & 73.7 & 69.1 & 57.2 & 31.8 \\
    4DMAP~\cite{liu2026pointnet4d}           & WACV26  & 72.4 & 78.7 & 78.8 & 74.8 & 63.6 & 31.8 \\
    \midrule
    \rowcolor{lightblue}
    P4Transformer + DiMP (Ours) & \textbf{This Paper}
      & \textbf{80.35}\,\scalebox{0.55}{\textcolor{magenta}{$(+13.65)$}}
      & \textbf{82.19}\,\scalebox{0.55}{\textcolor{magenta}{$(+20.19)$}}
      & \textbf{84.77}\,\scalebox{0.55}{\textcolor{magenta}{$(+19.47)$}}
      & \textbf{83.21}\,\scalebox{0.55}{\textcolor{magenta}{$(+23.41)$}}
      & \textbf{73.50}\,\scalebox{0.55}{\textcolor{magenta}{$(+27.20)$}}
      & \textbf{27.3}\,\scalebox{0.55}{\textcolor{magenta}{$(+0.7)$}} \\
    \bottomrule
  \end{tabularx}
\end{table*}
\paragraph{3D Action Recognition on MSRAction-3D.}
\begin{wrapfigure}{r}{0.5\linewidth}
      \vspace{-0.5\baselineskip}
      \centering
      \includegraphics[width=\linewidth]{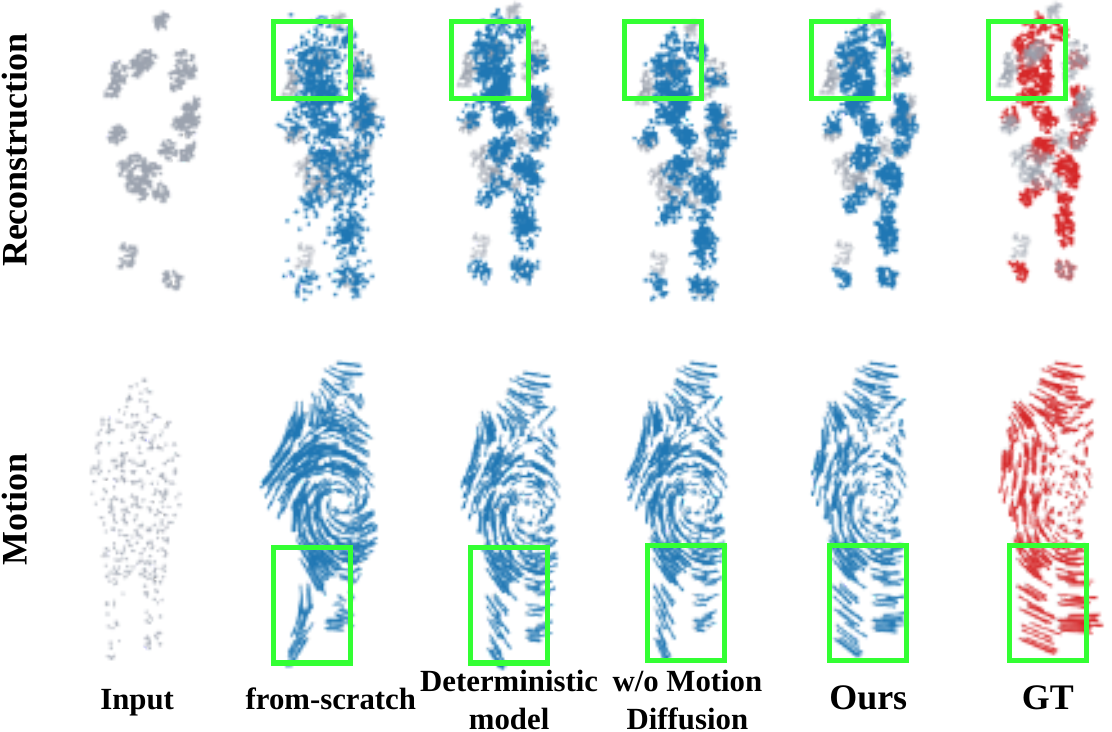}
      \caption{Qualitative comparison of geometric reconstruction (top) and inter-frame
      displacement prediction (bottom). Columns show the input, a model trained from
      scratch, a deterministic motion-supervision baseline, DiMP without motion diffusion,
      full DiMP, and ground truth (GT). Green boxes highlight local differences.
      Full DiMP yields predictions closer to ground truth, while removing motion diffusion
      degrades prediction quality.}
      \label{fig:ablation-vis}
    \end{wrapfigure}
As shown in Table~\ref{tab:msr}, DiMP outperforms methods that rely on deterministic motion objectives and achieves competitive performance against state-of-the-art pre-training methods on MSRAction-3D. DiMP yields larger gains on the weaker P4Transformer than on PST-Transformer, consistent with the backbone-dependent pattern observed in action segmentation. This reflects that P4Transformer's simpler temporal modeling leaves greater representational headroom for DiMP's distributional motion supervision to fill, whereas PST-Transformer's stronger attention already captures a portion of the motion statistics that pretraining targets.

\paragraph{Online Inference on HOI4D.}
We evaluate DiMP under the causally constrained online inference protocol of
PointNet4D~\cite{liu2026pointnet4d}.
As shown in Table~\ref{tab:online}, DiMP consistently outperforms all baselines.
Notably, the online gains in Edit and F1@50 exceed the corresponding offline gains
(Table~\ref{tab:action-seg}) by 8.39 and 4.63 points, respectively, suggesting that a learned motion distribution provides an
especially strong predictive prior when the model must reason about future motion from
past context alone without access to future frames.
This is expected, as causal constraints force the model to predict future motion
from past context alone, a setting where a learned motion distribution provides
a stronger prior than a deterministic estimate.

Further dataset descriptions, evaluation splits, task protocols, and training setups for all downstream experiments above are summarized in Appendix~\ref{app:exp-details}.

\subsection{Ablation studies}
\label{sec:ablation}

All ablation experiments use P4Transformer as the backbone under the same pre-training and fine-tuning setup summarized in Appendix~\ref{app:exp-details}.

\paragraph{Effect of Visible Center Corruption.}

We apply Gaussian noise at varying relative scales to visible tube centers and measure
the effect on downstream performance (Table~\ref{tab:ablation-noise}).
Performance degrades monotonically as visible center corruption increases, confirming
that corrupting visible coordinates undermines the motion-continuity prior and impairs
cross-frame geometric association.

\paragraph{Factorial interaction of center and motion diffusion.}
Table~\ref{tab:ablation-prerequisite} shows that motion diffusion alone yields near-zero gain, whereas enabling center diffusion first allows motion diffusion to contribute substantially, a pattern that holds consistently across both HOI4D and MSRAction-3D, confirming that center diffusion is a structural prerequisite rather than a dataset-specific effect. Further theoretical analysis is provided in Appendix~\ref{app:leakage-prerequisite}.
\paragraph{Progressive Component Analysis.}

Table~\ref{tab:ablation-prog} shows a progressive ablation from the
MaST-Pre~\cite{shen2023masked} baseline. Center diffusion gives the clearest structural
improvement, suggesting that removing positional leakage forces the model to
infer masked tube locations from context. Patch diffusion and motion diffusion
further strengthen the representation, with the full stratified variant
achieving the best overall trade-off.
Figure~\ref{fig:ablation-vis} provides a qualitative comparison of geometric
reconstruction and inter-frame displacement prediction under different training strategies.

\begin{table*}[t]
  \centering
  \begin{minipage}[t]{0.33\textwidth}
    \centering
    \captionof{table}{Effect of applying Gaussian noise to visible tube centers at varying
    relative strengths. Noise level 0.0 corresponds to the DiMP design, and 1.0 applies the same schedule as masked centers.}
    \label{tab:ablation-noise}
    \footnotesize
    \setlength{\tabcolsep}{3pt}
    \begin{tabularx}{\linewidth}{Xcc}
      \toprule
      Visible Noise Level & F1@50 $\uparrow$ & mIoU $\uparrow$ \\
      \midrule
      \rowcolor{lightblue}
      0.0 (DiMP)          & \textbf{80.77}    & \textbf{47.6}   \\
      0.25                & 77.91             & 43.2            \\
      0.50                & 75.43             & 41.9            \\
      1.0                 & 74.16             & 39.8            \\
      \bottomrule
    \end{tabularx}
  \end{minipage}\hfill
  \begin{minipage}[t]{0.64\textwidth}
    \centering
    \captionof{table}{Factorial ablation of center diffusion and motion diffusion on HOI4D action segmentation and MSRAction-3D action recognition. All other DiMP components are held fixed. The results are used to  verify the prerequisite role of center diffusion in enabling effective motion modeling.
    }
    \label{tab:ablation-prerequisite}
    \footnotesize
    \setlength{\tabcolsep}{6pt}
    \begin{tabularx}{\linewidth}{cc *{3}{>{\centering\arraybackslash}X}}
      \toprule
      \multirow{2}{*}{CD} & \multirow{2}{*}{MD}
        & \multicolumn{2}{c}{HOI4D} & MSRAction-3D \\
      \cmidrule(lr){3-4}\cmidrule(l){5-5}
        & & F1@50\,$\uparrow$ & mIoU\,$\uparrow$ & Acc\,$\uparrow$ \\
      \midrule
      $\times$ & $\times$         & 75.37 & 41.3 & 91.63 \\
      $\times$ & $\checkmark$     & 75.43 & 41.2 & 91.72 \\
      \midrule
      $\checkmark$ & $\times$     & 77.28 & 43.5 & 92.46 \\
      \rowcolor{lightblue}$\checkmark$ & $\checkmark$ & \textbf{80.77} & \textbf{47.6} & \textbf{93.97} \\
      \bottomrule
    \end{tabularx}
  \end{minipage}
\end{table*}

\begin{table*}[t]
  \centering
  \caption{Progressive component ablation on HOI4D. Each row adds one component
    to the previous configuration. VM: VisMaskTransformer encoder; PD: patch diffusion
    decoder; CD: masked-only center diffusion.}
  \label{tab:ablation-prog}
  \footnotesize
  \renewcommand{\arraystretch}{1.00}
  \setlength{\tabcolsep}{12pt}
  \begin{tabularx}{\textwidth}{
      >{\raggedright\arraybackslash}X
      c c c c c}
    \toprule
    Configuration
      & VM\enspace PD\enspace CD
      & Motion
      & F1@50 $\uparrow$
      & mIoU $\uparrow$ \\
    \midrule
    MaST-Pre baseline$^\dagger$
      & $\times$\enspace$\times$\enspace$\times$         & Deterministic Regression       & 63.84 & 40.3 \\
    + Joint linear PE
      & $\times$\enspace$\times$\enspace$\times$         & Deterministic Regression       & 65.13 & 40.4 \\
    + VisMask encoder
      & $\checkmark$\enspace$\times$\enspace$\times$     & Deterministic Regression       & 65.37 & 41.3 \\
    + Center diffusion
      & $\checkmark$\enspace$\times$\enspace$\checkmark$ & Deterministic Regression       & 72.28 & 43.5 \\
    + Patch diffusion
      & $\checkmark$\enspace$\checkmark$\enspace$\checkmark$ & Deterministic Regression   & 74.17 & 44.1 \\
    + Single-step motion
      & $\checkmark$\enspace$\checkmark$\enspace$\checkmark$ & Single-step DDPM & 78.41 & 46.1 \\
    \rowcolor{lightblue}
    DiMP full ($h{=}4$)
      & $\checkmark$\enspace$\checkmark$\enspace$\checkmark$ & Stratified DDPM & \textbf{80.77} & \textbf{47.6} \\
    \bottomrule
  \end{tabularx}
\end{table*}

Additional ablations on hyperparameter sensitivity, the stop-gradient mechanism,
and correspondence quality are provided in Appendix~\ref{app:more-ablation}.
Gradient propagation analysis is provided in Appendix~\ref{app:encoder-gradient}.

\section{Conclusion}

In this paper, we present DiMP, the first pretraining framework for dynamic point clouds that leverages DDPMs for distributional motion modeling. By reformulating inter-frame displacement supervision as a noise prediction task conditioned on decoded representations, DiMP targets the full conditional distribution over motion trajectories under a variational surrogate, rather than collapsing to a deterministic mean estimate. A key enabler is the center diffusion strategy, which avoids spatio-temporal positional leakage. Experiments on HOI4D and MSRAction-3D demonstrate consistent improvements over prior methods across multiple tasks.
One limitation of DiMP is backbone efficiency. Our implementation relies on
P4Transformer, whose quadratic self-attention complexity results in lower online
throughput than linear-complexity architectures.
As an initial exploration of distributional motion modeling for dynamic point
cloud pretraining, we expect DiMP to serve as a reference baseline and to inspire
further research on distributional representation learning for dynamic point clouds.

\begin{ack}
This work was supported in part by NSFC under Grant Nos. 62125305 and 62376206,
the Natural Science Basis Research Plan in Shaanxi Province of China under Grant No. 2025JC-JCQN-091,
and the Technology Innovation Leading Program of Shaanxi (Program No. 2024QY-SZX-23).
\end{ack}

\newpage
{\small
\bibliographystyle{unsrtnat}
\bibliography{references}
}

\newpage
\appendix
\section*{Appendix table of contents}
\begingroup
\hypersetup{linkcolor=black}
\startcontents[appendix]
\printcontents[appendix]{l}{1}{\setcounter{tocdepth}{2}}
\endgroup
\newpage

\section{Preliminaries}
\label{app:prelim}

\paragraph{Dynamic point cloud and point tube construction.}
Let $\mathcal{P} = \{\mathbf{P}^1, \mathbf{P}^2, \ldots, \mathbf{P}^L\} \in
\mathbb{R}^{L \times N \times 3}$ denote a dynamic point cloud of $L$ frames,
where each frame $\mathbf{P}^l$ contains $N$ points with $(x,y,z)$ coordinates.
Following MaST-Pre~\cite{shen2023masked}, we segment $\mathcal{P}$ into $K$
spatio-temporal point tubes.
Specifically, $K$ keypoints $\{\hat{\mathbf{p}}_i\}_{i=1}^{K}$ are sampled via
Farthest Point Sampling (FPS) applied to the first frame, where the $i$-th
tube collects all points within spatial radius $r$ and temporal distance
$l/2$ of keypoint $\hat{\mathbf{p}}_i$:
\begin{equation}
  \mathrm{Tube}_{\hat{\mathbf{p}}_i}
  = \bigl\{\mathbf{p} \mid \mathbf{p} \in \mathcal{P},\;
           D_s(\mathbf{p}, \hat{\mathbf{p}}_i) < r,\;
           D_t(\mathbf{p}, \hat{\mathbf{p}}_i) < \tfrac{l}{2}\bigr\}.
  \label{eq:tube-app}
\end{equation}
Each tube is centered at its keypoint coordinate $\hat{\mathbf{p}}_i$,
referred to as the tube center.
After random masking with ratio $\rho$, the tubes are partitioned into a
visible set $\mathcal{V}$ of size $K_v = \lfloor K(1-\rho) \rfloor$ and a
masked set $\mathcal{M}$ of size $K_m = K - K_v$, with centers
$\mathbf{C}^v \in \mathbb{R}^{K_v \times 3}$ and
$\mathbf{C}^m \in \mathbb{R}^{K_m \times 3}$.

\paragraph{Denoising diffusion probabilistic models.}
DDPMs~\cite{ho2020denoising} define a forward Markov chain that progressively
corrupts a clean signal $\mathbf{x}_0$ with Gaussian noise over $T$ steps:
\begin{equation}
  q(\mathbf{x}_t \mid \mathbf{x}_0)
  = \mathcal{N}\!\left(\mathbf{x}_t;\, \sqrt{\bar{\alpha}_t}\,\mathbf{x}_0,\,
    (1 - \bar{\alpha}_t)\mathbf{I}\right),
  \label{eq:ddpm-forward-app}
\end{equation}
where $\bar{\alpha}_t = \prod_{s=1}^{t}(1 - \beta_s)$ and
$\{\beta_t\}_{t=1}^{T}$ is a pre-defined noise schedule.
A neural network $\epsilon_\theta$ is trained to predict the added noise via
\begin{equation}
  \mathcal{L}_{\mathrm{DDPM}}
  = \mathbb{E}_{t,\,\mathbf{x}_0,\,\boldsymbol{\epsilon}}\!
    \left[\bigl\|
      \boldsymbol{\epsilon} -
      \epsilon_\theta(\mathbf{x}_t,\, t,\, \mathbf{c})
    \bigr\|^2\right],
  \quad \boldsymbol{\epsilon} \sim \mathcal{N}(\mathbf{0}, \mathbf{I}),
  \label{eq:ddpm-loss-app}
\end{equation}
where $\mathbf{c}$ is an optional conditioning signal.
Both diffusion modules in DiMP instantiate this framework with
task-specific targets and conditioning.

\section{Experimental details}
\label{app:exp-details}

\paragraph{Datasets.}
We evaluate DiMP on four benchmark datasets spanning hand-object interaction,
full-body action, and hand gesture understanding.
\textbf{HOI4D}~\cite{liu2022hoi4d} is an egocentric dataset comprising
3{,}863 dynamic point cloud collected from 9 participants interacting with
800 distinct object instances across 610 indoor environments, with frame-level
action labels across 19 action classes for action segmentation and point-level
labels across 39 semantic categories for semantic segmentation.
We follow the official split of 2{,}971 training scenes and 892 test scenes.
\textbf{MSRAction-3D}~\cite{li2010action} consists of 567 human point cloud
videos spanning 20 action categories, with 270 videos used for training and
297 for testing.
\textbf{SHREC'17}~\cite{desmedt2017shrec} is a gesture recognition benchmark
comprising 2{,}800 videos across 28 gesture classes, split into 1{,}960
training and 840 test videos.
\textbf{NvGesture}~\cite{molchanov2016online} consists of 1{,}532 videos
covering 25 gesture classes, with 1{,}050 training and 482 test videos.

\paragraph{Common pre-training setup.}
Following MaST-Pre~\cite{shen2023masked} and
M2PSC~\cite{han2024masked}, we employ
P4Transformer~\cite{fan2021point} as the backbone encoder.
During pre-training on HOI4D, 24 frames are densely sampled and 1{,}024
points are selected per frame with a temporal downsampling rate of 2 and a
temporal kernel size $l = 3$ per point tube.
The spatial downsampling rate is set to 32, the spatial neighborhood radius is
$r = 0.1$, and 32 neighbor points are sampled within each spherical query.
The masking ratio is set to 60\%~(Table~\ref{tab:ablation-masking}).
DiMP is pre-trained for 200 epochs using AdamW with an initial learning rate
of $1{\times}10^{-3}$ and cosine decay scheduling, with linear warmup for the
first 10 epochs and a batch size of 128.
All pre-training and fine-tuning experiments are conducted on 8 NVIDIA A100
GPUs.
After pre-training, the decoder and both diffusion heads are discarded and only
the encoder is retained and fine-tuned for each downstream task.

\paragraph{Motion target construction.}
Inter-frame displacement targets $\mathbf{M}_0$ are pre-computed offline using the
frozen CorrNet3D~\cite{zeng2021corrnet3d} correspondence model before training begins.
CorrNet3D is invoked once per dataset split to generate displacement files, which are
stored on disk and loaded as static supervision labels during pretraining, in the same
way that HOI4D action labels are loaded.
No gradient flows through CorrNet3D at any point, and the correspondence files are
discarded at inference time along with all other pretraining-only modules.

\paragraph{HOI4D action segmentation.}
For offline action segmentation, each sequence contains 150 frames and each
frame contains 2{,}048 points.
The model predicts an action label for every frame.

\paragraph{HOI4D semantic segmentation.}
For semantic segmentation, the full HOI4D dataset is used, where each sequence
contains 300 frames of point clouds with 8{,}192 points per frame annotated
across 39 semantic categories.
During pre-training, the sequence length is set to 10 with 4{,}096 points per
frame.
Fine-tuning and testing are performed on sequences of length 3, consistent with
prior methods.

\paragraph{MSRAction-3D action recognition.}
Following established convention~\cite{shen2023masked,han2024masked}, 270
videos are used for training and 297 for testing.
During fine-tuning, 24 frames are densely sampled and 2{,}048 points are
selected per frame.
AdamW with cosine decay is used for optimization.

\paragraph{HOI4D online inference.}
For online action segmentation, we follow the protocol of
PointNet4D~\cite{liu2026pointnet4d}.
The model processes the dynamic point cloud sequentially and is permitted to
access only the current and past frames at inference time.
To ensure causal consistency, P4Transformer is equipped with a causal attention
mask during online inference.
All methods operate on sequences of 150 frames.
We report both task performance metrics (Acc, Edit, F1@10, F1@25, F1@50) and
computational efficiency (Clips/s).

\paragraph{Gesture recognition fine-tuning (NvGesture and SHREC'17).}
For the transfer-to-gesture-recognition experiments in Appendix~\ref{app:hoi4d-gesture},
all encoders are pre-trained on HOI4D following the common pre-training setup described above,
and then fine-tuned on NvGesture and SHREC'17 under an end-to-end fine-tuning setting.
During fine-tuning, 24 frames are densely sampled and 1{,}024 points are selected per frame,
consistent with the HOI4D pre-training configuration.
The AdamW optimizer is used with a cosine decay strategy.
Following MaST-Pre~\cite{shen2023masked} and M2PSC~\cite{han2024masked},
the initial learning rate is set to $1{\times}10^{-3}$ for NvGesture
and $5{\times}10^{-4}$ for SHREC'17, with a batch size of 24.
We report recognition accuracy (\%) at both 30 and 50 fine-tuning epochs.
The reproduced MaST-Pre$^\dagger$ and M2PSC$^\dagger$ baselines are fine-tuned
under identical settings to ensure a fair comparison.

\section{Additional ablation studies}
\label{app:more-ablation}

\paragraph{Hyperparameter sensitivity.}
\label{app:hyper-sensitivity}
\label{app:masking-ratio}
Tables~\ref{tab:ablation-intervals},~\ref{tab:ablation-weight-cen}, and~\ref{tab:ablation-weight-mot}
report sensitivity to the number of stratified motion intervals $h$, the center-loss weight
$\gamma_{\mathrm{cen}}$, and the motion-loss weight $\lambda_{\mathrm{mot}}$.
Performance peaks at $h{=}4$, where more intervals enrich hierarchical supervision,
while $h{=}8$ slightly hurts due to narrower per-interval ranges.
DiMP is moderately sensitive to $\gamma_{\mathrm{cen}}$ and $\lambda_{\mathrm{mot}}$,
with optima at $0.1$ and $1.0$ respectively, and degrades smoothly elsewhere.
Table~\ref{tab:ablation-masking} reports downstream performance as the tube masking
ratio $\rho$ is varied from $50\%$ to $80\%$.
DiMP adopts $\rho{=}60\%$ by default, at which performance peaks.
Reducing $\rho$ below $60\%$ leaves too few masked tubes to provide sufficient
pretraining difficulty, while increasing $\rho$ beyond $60\%$ gradually degrades
performance as the decoder has less visible context for spatio-temporal inference.
Crucially, even at $\rho{=}75\%$, the ratio adopted by MaST-Pre~\cite{shen2023masked}
and M2PSC~\cite{han2024masked} for a directly comparable setup, DiMP achieves
$79.45$ F1@50 and $46.5$ mIoU, exceeding the best prior method
(M2PSC, $65.9$ F1@50 and $42.3$ mIoU) by clear margins of $+13.55$ and $+4.2$
respectively.
This confirms that DiMP's improvements are not an artifact of a lower masking
rate but arise from its distributional motion modeling and leakage-free
center diffusion design.

\begin{table*}[t]
  \centering
  \small
  \begin{minipage}[t]{0.32\linewidth}
    \centering
    \captionof{table}{Sensitivity to the number of stratified intervals $h$.}
    \label{tab:ablation-intervals}
    \setlength{\tabcolsep}{5pt}
    \begin{tabularx}{\linewidth}{Xcc}
      \toprule
      $h$ & F1@50 $\uparrow$ & mIoU $\uparrow$ \\
      \midrule
      1        & 79.94           & 46.1 \\
      2        & 80.13           & 46.8 \\
      \rowcolor{lightblue}
      4 (DiMP) & \textbf{80.77} & \textbf{47.6} \\
      8        & 80.73           & 46.9 \\
      \bottomrule
    \end{tabularx}
  \end{minipage}\hfill
  \begin{minipage}[t]{0.32\linewidth}
    \centering
    \captionof{table}{Sensitivity to $\gamma_{\mathrm{cen}}$.
    $\lambda_{\mathrm{mot}}$ fixed at $1.0$.}
    \label{tab:ablation-weight-cen}
    \setlength{\tabcolsep}{4pt}
    \begin{tabularx}{\linewidth}{Xc}
      \toprule
      $\gamma_{\mathrm{cen}}$ & F1@50 $\uparrow$ \\
      \midrule
      0.05            & 79.11          \\
      \rowcolor{lightblue}
      0.1             & \textbf{80.77} \\
      0.5             & 80.12          \\
      1.0             & 78.97          \\
      \bottomrule
    \end{tabularx}
  \end{minipage}\hfill
  \begin{minipage}[t]{0.32\linewidth}
    \centering
    \captionof{table}{Sensitivity to $\lambda_{\mathrm{mot}}$.
    $\gamma_{\mathrm{cen}}$ fixed at $0.1$.}
    \label{tab:ablation-weight-mot}
    \setlength{\tabcolsep}{4pt}
    \begin{tabularx}{\linewidth}{Xc}
      \toprule
      $\lambda_{\mathrm{mot}}$ & F1@50 $\uparrow$ \\
      \midrule
      0.25            & 78.34          \\
      0.5             & 79.15          \\
      \rowcolor{lightblue}
      1.0             & \textbf{80.77} \\
      2.0             & 80.64          \\
      \bottomrule
    \end{tabularx}
  \end{minipage}
\end{table*}

\begin{table}[t]
  \caption{Sensitivity of DiMP to the tube masking ratio $\rho$ on HOI4D
  action segmentation and semantic segmentation.
  The row marked with $*$ uses the same masking ratio as
  MaST-Pre~\cite{shen2023masked} and M2PSC~\cite{han2024masked} for a
  controlled comparison. All other DiMP settings are fixed.}
  \label{tab:ablation-masking}
  \centering
  \footnotesize
  \setlength{\tabcolsep}{8pt}
  \begin{tabular}{lcc}
    \toprule
    Masking Ratio $\rho$ & F1@50 $\uparrow$ & mIoU $\uparrow$ \\
    \midrule
    50\%              & 79.31          & 46.2          \\
    \rowcolor{lightblue}
    60\% (DiMP)       & \textbf{80.77} & \textbf{47.6} \\
    70\%              & 79.88          & 46.8          \\
    75\%$^*$          & 79.45          & 46.5          \\
    80\%              & 78.93          & 45.9          \\
    \bottomrule
  \end{tabular}
\end{table}

\paragraph{Stop-Gradient on the Center Predictor.}
The stop-gradient operator $\mathrm{sg}(\cdot)$ applied to $\hat{\mathbf{C}}_0^m$
before it enters the decoder (Eq.~\eqref{eq:decoder}) is designed to decouple
$\theta_{\mathrm{cen}}$ from $\mathcal{L}_{\mathrm{geo}}$ and $\mathcal{L}_{\mathrm{mot}}$,
ensuring that the Center Predictor $g_m$ is updated exclusively by the center
denoising objective $\mathcal{L}_{\mathrm{cen}}$.
Table~\ref{tab:ablation-sg} ablates the two gradient paths independently by
selectively removing the stop-gradient from each loss.

\textit{w/o sg (both).}
Removing $\mathrm{sg}(\cdot)$ entirely allows both $\mathcal{L}_{\mathrm{geo}}$ and
$\mathcal{L}_{\mathrm{mot}}$ to back-propagate through $g_m$,
yielding the largest performance drop.
The reconstruction objective drives $g_m$ to produce centers that minimize
Chamfer distance rather than noise prediction error, corrupting the positional
inference signal.
The motion objective compounds this by injecting an
additional conflicting gradient.

\textit{sg on $\mathcal{L}_{\mathrm{geo}}$ only.}
Blocking only the geometric gradient leaves $\mathcal{L}_{\mathrm{mot}}$ free to
interfere with $g_m$.
Performance recovers partially but remains below the full DiMP setting,
confirming that the motion diffusion gradient also distorts center denoising
when left unblocked.

\textit{sg on $\mathcal{L}_{\mathrm{mot}}$ only.}
Blocking only the motion gradient while permitting the geometric gradient
likewise falls short of the full design.
The drop is slightly larger than the symmetric case above, consistent with
$\mathcal{L}_{\mathrm{geo}}$ having a shorter gradient path to $g_m$ through the
decoder and thus exerting stronger interference.

\textit{DiMP (sg on both).}
Blocking both gradient paths achieves the best performance on both metrics,
confirming that objective isolation is a necessary condition for reliable
positional inference.

\begin{table}[h]
  \caption{Ablation of the stop-gradient mechanism on HOI4D.
  ``sg($\mathcal{L}_{\mathrm{geo}}$)'' and ``sg($\mathcal{L}_{\mathrm{mot}}$)''
  indicate whether the gradient of each loss is blocked from entering $g_m$.
  All other settings follow the complete DiMP configuration.}
  \label{tab:ablation-sg}
  \centering
  \footnotesize
  \setlength{\tabcolsep}{6pt}
  \begin{tabular}{lcccc}
    \toprule
    Variant
      & sg($\mathcal{L}_{\mathrm{geo}}$)
      & sg($\mathcal{L}_{\mathrm{mot}}$)
      & F1@50 $\uparrow$
      & mIoU $\uparrow$ \\
    \midrule
    w/o sg (both)
      & $\times$ & $\times$ & 77.31 & 43.8 \\
    sg on $\mathcal{L}_{\mathrm{geo}}$ only
      & $\checkmark$ & $\times$ & 79.15 & 45.9 \\
    sg on $\mathcal{L}_{\mathrm{mot}}$ only
      & $\times$ & $\checkmark$ & 78.67 & 45.3 \\
    \rowcolor{lightblue}
    DiMP (sg on both)
      & $\checkmark$ & $\checkmark$ & \textbf{80.77} & \textbf{47.6} \\
    \bottomrule
  \end{tabular}
\end{table}

\paragraph{Sensitivity to Correspondence Quality (CorrNet3D Ablation).}
\label{app:corrnet-ablation}

To assess how sensitive DiMP is to the quality of the external correspondence model
used to generate motion targets, we evaluate downstream performance under five
motion-target variants on HOI4D action segmentation (F1@50 and mIoU).
All other DiMP settings are fixed.

\textit{Correspondence mechanisms.}
(a) \textbf{CorrNet3D} (default): the pre-trained correspondence model used throughout the paper.
(b) \textbf{$k$-NN}: each point is matched to its nearest neighbor in the next frame by
Euclidean distance, with no learned model.
(c) \textbf{CorrNet3D + 5\% noise}: 5\% of correspondences are randomly permuted before
computing displacements, simulating moderate correspondence error.
(d) \textbf{CorrNet3D + 20\% noise}: 20\% of correspondences are randomly permuted,
simulating severe correspondence degradation.

\begin{table}[h]
  \caption{Sensitivity of DiMP to motion-target quality on HOI4D action segmentation.
  All variants use the complete DiMP architecture and only the correspondence mechanism
  for building $\mathbf{M}_0$ changes.
  $\Delta$ denotes the absolute F1@50 change relative to the DiMP default.}
  \label{tab:corrnet-ablation}
  \centering
  \footnotesize
  \setlength{\tabcolsep}{5pt}
  \begin{tabular}{lcccc}
    \toprule
    Motion Target & Correspondence Mechanism & F1@50 $\uparrow$ & mIoU $\uparrow$ & $\Delta$ F1@50 \\
    \midrule
    \rowcolor{lightblue}
    CorrNet3D (DiMP default)  & Learned correspondence     & \textbf{80.77} & \textbf{47.6} & 0.00 \\
    $k$-NN (no model)         & Euclidean nearest-neighbor & 79.12 & 45.7 & -1.65 \\
    CorrNet3D + 5\% noise     & + random permutation       & 80.43 & 47.3 & -0.34 \\
    CorrNet3D + 20\% noise    & + random permutation       & 79.48 & 46.3 & -1.29 \\
    No motion branch          & none                       & 77.28 & 43.5 & -3.49 \\
    \bottomrule
  \end{tabular}
\end{table}

The key comparison is between $k$-NN and CorrNet3D. If the zero-training-cost
$k$-NN mechanism still enables DiMP to substantially outperform the no-motion-branch
baseline, it establishes that the gains arise from the diffusion-based pretext
objective itself rather than from representational capacity imported from CorrNet3D.
The noise-injection variants further quantify the graceful degradation curve,
showing at what correspondence error rate the gains are materially diminished.

\section{Additional experimental results}
\label{app:more-results}

\subsection{Transfer to gesture recognition (HOI4D pretraining)}
\label{app:hoi4d-gesture}

We report an \emph{unofficial} supplementary evaluation in which all encoders are instead
pre-trained on HOI4D and subsequently fine-tuned on NvGesture and SHREC'17 for gesture recognition.
To our knowledge, no prior work has reported results under this protocol, as the
official codebases of existing methods do not provide pretraining configurations for the HOI4D
dataset.
To enable a fair comparison, we independently re-implemented the pretraining procedures of
MaST-Pre~\cite{shen2023masked} and M2PSC~\cite{han2024masked} following their respective papers
as closely as possible. Results obtained under these reproduced configurations are marked
with~$^\dagger$ in Table~\ref{tab:hoi4d-gesture}.
We report recognition accuracy (\%) after fine-tuning for 30 and 50 epochs, following the same
fine-tuning setup as Appendix~\ref{app:exp-details}.
As shown in Table~\ref{tab:hoi4d-gesture}, DiMP outperforms all reproduced
baselines at both fine-tuning budgets on both datasets.
At the 30-epoch budget, DiMP improves over the from-scratch P4Transformer
by $+1.4$ on NvGesture and $+2.3$ on SHREC'17, and over the strongest reproduced
baseline (M2PSC$^\dagger$) by $+0.4$ and $+0.3$ respectively.
At the 50-epoch budget, the gains over P4Transformer are $+1.8$ on NvGesture
and $+1.5$ on SHREC'17, while the margins over M2PSC$^\dagger$ grow to $+0.7$ and
$+0.9$, indicating that DiMP's distributional motion objective yields
representations that remain competitive and even improve relative to
deterministic baselines as fine-tuning proceeds.

\begin{table}[h]
  \caption{Gesture recognition accuracy (\%) on NvGesture (NvG) and SHREC'17 (SHR) when
  all encoders are pre-trained on HOI4D. Methods marked with~$^\dagger$ are re-implemented by us, as
  no official HOI4D pretraining scripts are publicly available for these methods.}
  \label{tab:hoi4d-gesture}
  \centering
  \small
  \begin{tabular}{lcccc}
    \toprule
    \multirow{2}{*}{Method} & \multicolumn{2}{c}{30 Epochs} & \multicolumn{2}{c}{50 Epochs} \\
    \cmidrule(lr){2-3}\cmidrule(lr){4-5}
    & NvG & SHR & NvG & SHR \\
    \midrule
    P4Transformer~\cite{fan2021point}             & 83.7 & 85.9 & 86.1 & 89.9 \\
    P4Transformer + MaST-Pre$^\dagger$~\cite{shen2023masked} & 84.4 & 87.3 & 86.7 & 90.1 \\
    P4Transformer + M2PSC$^\dagger$~\cite{han2024masked}      & 84.7 & 87.9 & 87.2 & 90.5 \\
    \midrule
    \rowcolor{lightblue}
    P4Transformer + DiMP (Ours) & \textbf{85.1} & \textbf{88.2} & \textbf{87.9} & \textbf{91.4} \\
    \bottomrule
  \end{tabular}
\end{table}

\subsection{Qualitative visualization of point cloud reconstruction}
\label{app:recon-vis}

Figure~\ref{fig:recon-vis} provides qualitative visualizations of the dynamic point cloud
reconstructed by DiMP during the pretraining stage.
Each row shows a multi-frame input sequence alongside the corresponding reconstruction produced by
the masked decoder.
The results demonstrate that DiMP recovers accurate per-frame geometry across diverse object
interactions and viewpoints, confirming that the probabilistic motion diffusion objective does not
impair the geometric reconstruction quality of the backbone while simultaneously enabling
distributional motion modeling.

\begin{figure}[!h]
      \centering
      \includegraphics[width=\linewidth]{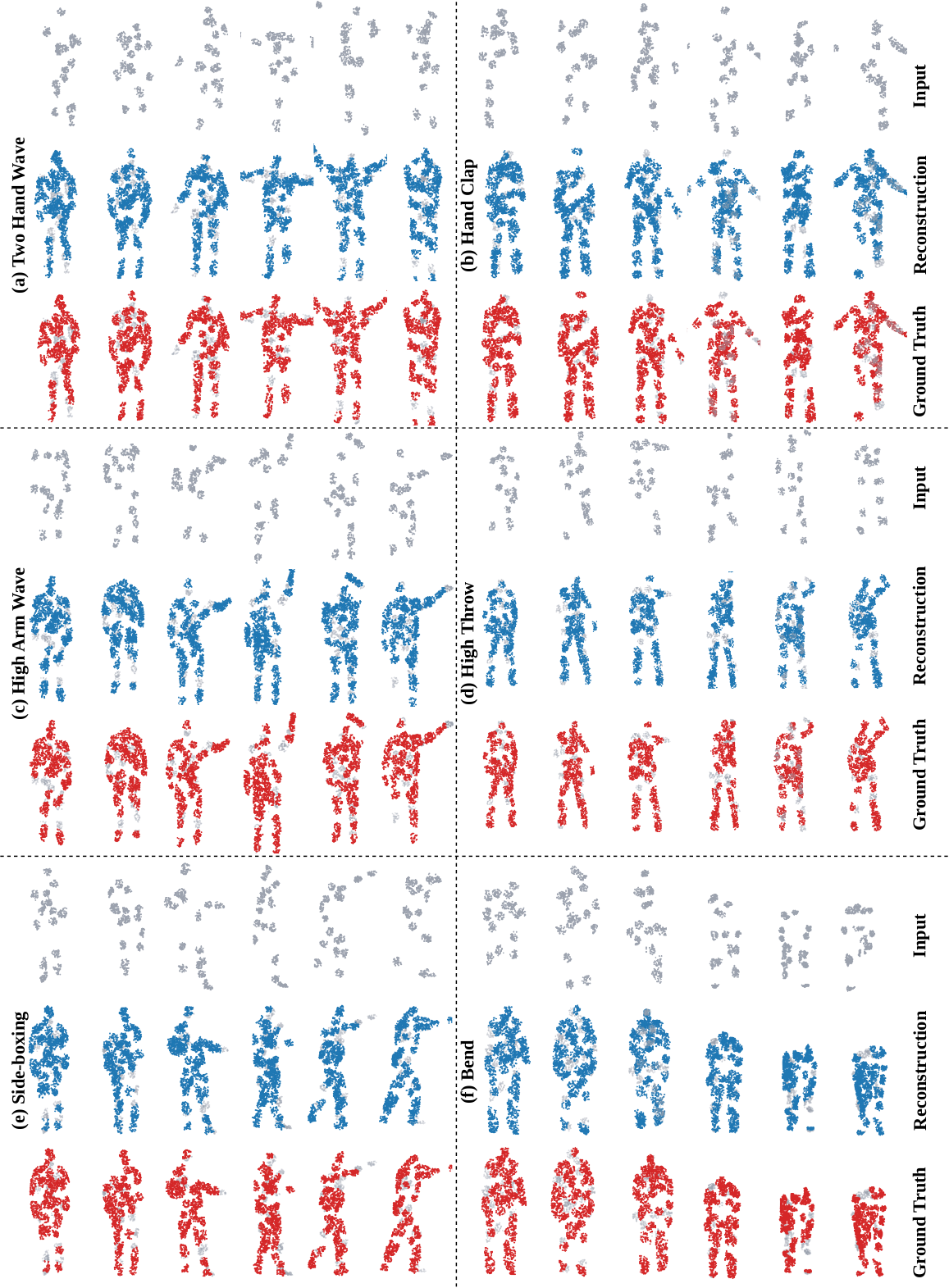}
      \caption{Qualitative visualization of dynamic point cloud reconstruction.
      Each example shows the masked input (top), reconstruction produced by DiMP's
      masked decoder (middle), and ground truth (bottom).
      The results illustrate that DiMP faithfully recovers fine-grained per-frame geometry
      across diverse action sequences, despite the high masking ratio (60\%)
      applied during pretraining.}
      \label{fig:recon-vis}
    \end{figure}

\subsection{Variance reporting and multi-seed statistics}
\label{app:variance}

To quantify result stability, DiMP and M2PSC are each evaluated over 3 independent
random seeds, with pretraining seed and fine-tuning seed held together per run.
Table~\ref{tab:multiseed} reports mean\,$\pm$\,std for the two action segmentation
benchmarks.

Across all tasks, DiMP's standard deviation does not exceed 0.7 points, indicating
stable convergence across seeds.
Crucially, every absolute improvement of DiMP over M2PSC exceeds the combined
standard deviations by a factor of at least $4\times$, confirming that the reported
gains are not attributable to random seed variance.

\begin{table}[h]
  \caption{Multi-seed statistics (mean $\pm$ std, 3 seeds) for DiMP and M2PSC on
  HOI4D offline and online action segmentation.
  Seed indices match across both methods for fair comparison.}
  \label{tab:multiseed}
  \centering
  \footnotesize
  \setlength{\tabcolsep}{4pt}
  \begin{tabular}{llccccc}
    \toprule
    Task & Method & Acc & Edit & F1@10 & F1@25 & F1@50 \\
    \midrule
    \multirow{2}{*}{HOI4D offline seg}
      & M2PSC & 75.9{\scriptsize\,±\,0.4} & 77.1{\scriptsize\,±\,0.5} & 78.4{\scriptsize\,±\,0.4} & 74.1{\scriptsize\,±\,0.5} & 65.9{\scriptsize\,±\,0.5} \\
      & DiMP  & 82.4{\scriptsize\,±\,0.5} & 84.9{\scriptsize\,±\,0.6} & 87.3{\scriptsize\,±\,0.5} & 85.9{\scriptsize\,±\,0.6} & 80.8{\scriptsize\,±\,0.5} \\
    \midrule
    \multirow{2}{*}{HOI4D online seg}
      & M2PSC & 73.2{\scriptsize\,±\,0.5} & 74.8{\scriptsize\,±\,0.6} & 76.1{\scriptsize\,±\,0.5} & 72.3{\scriptsize\,±\,0.6} & 63.1{\scriptsize\,±\,0.6} \\
      & DiMP  & 80.4{\scriptsize\,±\,0.6} & 82.2{\scriptsize\,±\,0.7} & 84.8{\scriptsize\,±\,0.6} & 83.2{\scriptsize\,±\,0.7} & 73.5{\scriptsize\,±\,0.6} \\
    \bottomrule
  \end{tabular}
\end{table}

\section{Positional leakage as a prerequisite barrier to distributional motion modeling}
\label{app:leakage-prerequisite}

\paragraph{Overview.}
The motion diffusion branch of DiMP is conditioned on the decoded representation
$\mathbf{Z}_{\mathrm{dec}}$ produced by the masked decoder.
For this branch to capture the full multimodal distribution of inter-frame motion
trajectories, $\mathbf{Z}_{\mathrm{dec}}$ must encode rich, globally coherent motion
context rather than merely local geometric features.
We show here, both analytically and empirically, that positional leakage,
which injects ground-truth masked tube centers $\mathbf{c}^*$ as decoder positional
embeddings, collapses decoder cross-attention into local geometric retrieval,
severing the informational pathway through which the motion diffusion loss drives
the encoder to internalize motion distributions.
Consequently, eliminating positional leakage is a \emph{necessary precondition}
for the motion diffusion branch to function as a genuine distributional learner.

\paragraph{Notation.}
Let $\mathbf{Z}_{\mathrm{enc}} = f_{\mathrm{enc}}(\mathbf{X})$ denote the encoder output
from visible tokens $\mathbf{X}$.
The masked decoder produces, for each masked tube $k$,
\begin{equation}
  \mathbf{Z}_{\mathrm{dec},k}
  = \mathrm{CrossAttn}\!\left(W_Q\bigl(\mathbf{m}_k + \mathrm{PE}(\mathbf{p}_k)\bigr),\;
    W_K \mathbf{Z}_{\mathrm{enc}},\; W_V \mathbf{Z}_{\mathrm{enc}}\right),
  \label{eq:dec-crossattn}
\end{equation}
where $\mathbf{m}_k$ is a learnable masked token embedding and $\mathbf{p}_k$ is the
positional input supplied to the decoder query.
Under \textbf{positional leakage}, $\mathbf{p}_k = \mathbf{c}^*_k$ is the ground-truth center.
Under \textbf{DiMP's center diffusion}, $\mathbf{p}_k = \hat{\mathbf{c}}_k$ is the predicted center
with residual uncertainty.
Let $\mathbf{M}$ denote the inter-frame displacement field and $A$ the action class label.

\paragraph{Assumption 1 (Strong positional dominance).}
\emph{We consider the limiting regime in which
(a) the positional embedding norm dominates the learnable mask token, \emph{i.e.},
$\|W_Q \mathrm{PE}(\mathbf{c}^*_k)\| \gg \|W_Q \mathbf{m}_k\|$, and
(b) the encoder has already learned a position-aligned key mapping, in the
sense that the similarity gap
\begin{equation}
\gamma_k := \max_{j'}\,\langle W_Q \mathrm{PE}(\mathbf{c}^*_k),\,
            W_K \mathbf{Z}_{\mathrm{enc},j'}\rangle
          - \mathrm{second}_{j'}\!\langle W_Q \mathrm{PE}(\mathbf{c}^*_k),\,
            W_K \mathbf{Z}_{\mathrm{enc},j'}\rangle
\end{equation}
is strictly positive and grows with $\|\mathrm{PE}(\mathbf{c}^*_k)\|$.}

Assumption~1(a) is a standard consequence of injecting ground-truth centers
into a trained Transformer, where positional features routinely exceed
content-token norms by an order of magnitude~\cite{he2022masked}.
Assumption~1(b) is empirically observable in any decoder trained with
positional supervision and is the source of the positional shortcut behavior
previously documented for masked
autoencoders~\cite{he2022masked,pang2022masked,zhang2022pointm2ae}.
Both conditions are \emph{regime assumptions}: they describe the limit in which
leakage is severe.
The following property is stated as a limit-case result under this regime, not
as a statement about the fully trained network at finite temperature.

\paragraph{Property 1 (Attention concentration in the leakage regime).}
\emph{Under Assumption~1 with $\tau = 1$ held fixed, the softmax attention
weight $\alpha_{kj}$ in Eq.~\eqref{eq:dec-crossattn} concentrates around the
single encoder token
$j^*_k := \arg\max_{j'}\,\langle W_Q \mathrm{PE}(\mathbf{c}^*_k),\,
W_K \mathbf{Z}_{\mathrm{enc},j'}\rangle$:}
\begin{equation}
  1 - \alpha_{k,j^*_k} \;\leq\; (T_{\mathrm{vis}}-1)\,e^{-\gamma_k},
  \qquad
  \alpha_{k,j^*_k} \;\to\; 1 \;\;\text{as}\;\; \gamma_k \to \infty.
  \label{eq:attn-collapse}
\end{equation}

\textit{Proof.}
Write the attention logits as $e_{kj} = \langle W_Q(\mathbf{m}_k +
\mathrm{PE}(\mathbf{c}^*_k)),\, W_K \mathbf{Z}_{\mathrm{enc},j}\rangle =
\tilde{e}_{kj} + \delta_{kj}$, where $\tilde{e}_{kj} = \langle W_Q
\mathrm{PE}(\mathbf{c}^*_k),\, W_K \mathbf{Z}_{\mathrm{enc},j}\rangle$ and
$|\delta_{kj}| \leq \|W_Q \mathbf{m}_k\|\,\|W_K \mathbf{Z}_{\mathrm{enc},j}\|$
is negligible under Assumption~1(a).
Under Assumption~1(b), $\tilde{e}_{k,j^*_k}$ exceeds every other
$\tilde{e}_{kj'}$ by at least $\gamma_k$, and the standard softmax tail bound
then yields $1 - \alpha_{k,j^*_k} \leq (T_{\mathrm{vis}}-1)\,e^{-\gamma_k}$.
\hfill$\square$

Property~1 recovers the one-hot collapse in the limit $\gamma_k \to \infty$
but is a finite-$\gamma_k$ concentration bound, which is the meaningful regime for
trained networks.
We next use it to argue that the decoder representation under leakage
effectively depends only on $\mathbf{Z}_{\mathrm{enc},j^*_k}$ up to a vanishing
residual.

Under Eq.~\eqref{eq:attn-collapse}, $\mathbf{Z}_{\mathrm{dec},k}^{\mathrm{leak}}
= W_V \mathbf{Z}_{\mathrm{enc},j^*_k} + O(e^{-\gamma_k})$ is a
\emph{near-single-token appearance summary} that carries only the local
geometric content of one encoder token up to an exponentially small residual
in $\gamma_k$, rather than the global spatio-temporal context required to model motion
distributions.

\paragraph{Corollary 1 (Motion information bottleneck in the strong leakage limit).}
\emph{In the limit $\gamma_k \to \infty$ of Assumption~1, the mutual information
between the inter-frame displacement $\mathbf{M}$ and the leaked decoder
representation is bounded as}
\begin{equation}
  I(\mathbf{M};\, \mathbf{Z}_{\mathrm{dec}}^{\mathrm{leak}})
  \;\leq\;
  I(\mathbf{M};\, \mathbf{Z}_{\mathrm{enc},j^*})
  \;\leq\;
  I(\mathbf{M};\, \mathbf{Z}_{\mathrm{enc}}).
  \label{eq:info-bottleneck}
\end{equation}
\emph{The first inequality is strict whenever $\mathbf{M}$ depends on encoder
tokens other than $j^*_k$, a condition that holds generically for any motion
involving multiple body parts or objects.}

\textit{Proof.}
In the limit $\gamma_k \to \infty$, the residual in the previous display
vanishes and $\mathbf{Z}_{\mathrm{dec},k}^{\mathrm{leak}}$ becomes a
deterministic function of $\mathbf{Z}_{\mathrm{enc},j^*_k}$ alone, so the
first inequality follows from the data processing inequality.
The second inequality holds because $\mathbf{Z}_{\mathrm{enc},j^*_k}$ is a
component of the full encoder representation $\mathbf{Z}_{\mathrm{enc}}$,
giving $I(\mathbf{M};\, \mathbf{Z}_{\mathrm{enc},j^*_k}) \leq
I(\mathbf{M};\, \mathbf{Z}_{\mathrm{enc}})$ by the chain rule and the
non-negativity of conditional mutual information.\hfill$\square$

\paragraph{Gradient flow analysis.}
During training, the motion diffusion loss gradient flows back to the encoder via the
decoder cross-attention:
\begin{equation}
  \frac{\partial \mathcal{L}_{\mathrm{mot}}}{\partial \mathbf{Z}_{\mathrm{enc}}}
  = \sum_k \frac{\partial \mathcal{L}_{\mathrm{mot}}}{\partial \mathbf{Z}_{\mathrm{dec},k}}
    \cdot \frac{\partial \mathbf{Z}_{\mathrm{dec},k}}{\partial \mathbf{Z}_{\mathrm{enc}}}.
\end{equation}
Under the attention concentration of Property~1, the Jacobian
$\partial \mathbf{Z}_{\mathrm{dec},k}/\partial \mathbf{Z}_{\mathrm{enc}}$
is dominated by its component at token $j^*_k$, with contributions from other
tokens suppressed by $O(e^{-\gamma_k})$: the gradient effectively reaches only
one encoder token per masked query, updating at most $K$ out of
$T_{\mathrm{vis}}$ visible tokens (where $T_{\mathrm{vis}} \gg K$ in practice).
This \emph{effective gradient sparsity} means the encoder learns to encode
only local retrieval-relevant features at specific spatial positions, not the
globally consistent motion-distribution context that the diffusion head
requires.
When positional leakage is removed, attention is distributed across all encoder
tokens (since $\hat{\mathbf{c}}_k$ carries residual uncertainty), the Jacobian is
dense, and the motion diffusion loss provides globally informative gradient signals
that drive the encoder to internalize the full motion distribution.

\begin{proposition}[Leakage caps action-discriminative information]
\label{prop:leakage-cap-discriminability}
Let $\mathbf{Z}_{\mathrm{enc}}^{\mathrm{leak},*}$ and
$\mathbf{Z}_{\mathrm{enc}}^{\mathrm{DiMP},*}$ be the encoder representations at the
population-optimum of the motion diffusion losses $\mathcal{L}_{\mathrm{mot}}^{\mathrm{leak}}$
and $\mathcal{L}_{\mathrm{mot}}^{\mathrm{DiMP}}$, respectively.
Assume:
\begin{enumerate}[label=(A\arabic*), leftmargin=2.2em, topsep=2pt, itemsep=1pt]
  \item \textbf{Attention concentration under leakage}: the leakage branch
    satisfies Assumption~1 in the limit $\gamma_k \to \infty$, so Corollary~1
    applies and
    $I(\mathbf{M};\mathbf{Z}_{\mathrm{dec}}^{\mathrm{leak}}) \leq
    I(\mathbf{M};\mathbf{Z}_{\mathrm{enc},j^*})$.
  \item \textbf{Markov structure}: the data process follows
    $A \to \mathbf{M} \to \mathbf{Z}_{\mathrm{enc}} \to \mathbf{Z}_{\mathrm{dec}}$,
    so the DPI gives $I(A;\mathbf{Z}_{\mathrm{enc}}) \leq I(A;\mathbf{M})$.
  \item \textbf{Global dependence of $A$ on $\mathbf{M}$}: the action label
    depends on displacement information from more than one spatial location,
    \emph{i.e.}, $I(A;\mathbf{M}) > I(A;\mathbf{M}_{j^*})$ where $\mathbf{M}_{j^*}$
    restricts $\mathbf{M}$ to the neighborhood of the bottleneck token
    $j^*$.
  \item \textbf{Variational tightness gap}: the DDPM loss is a variational
    upper bound on $-\mathbb{E}[\log p_{\theta}(\mathbf{M}\mid\mathbf{Z}_{\mathrm{dec}})]$,
    and the population optimum of the encoder under
    $\mathcal{L}_{\mathrm{mot}}^{\mathrm{DiMP}}$ attains a strictly tighter
    bound than under $\mathcal{L}_{\mathrm{mot}}^{\mathrm{leak}}$ when
    (A1)--(A3) hold.
\end{enumerate}
Then
\begin{equation}
  I\!\left(A;\; \mathbf{Z}_{\mathrm{enc}}^{\mathrm{leak},*}\right)
  \;<\;
  I\!\left(A;\; \mathbf{Z}_{\mathrm{enc}}^{\mathrm{DiMP},*}\right).
  \label{eq:joint-necessity}
\end{equation}
\end{proposition}

\begin{proof}
By (A4), at the population optimum the loss
$\mathcal{L}_{\mathrm{mot}}$ coincides with the conditional entropy
$H(\mathbf{M}\mid\mathbf{Z}_{\mathrm{dec}})$ (up to a constant in
$\mathbf{Z}_{\mathrm{dec}}$), so minimizing it is equivalent to maximizing
$I(\mathbf{M};\mathbf{Z}_{\mathrm{dec}})$ as a functional of the encoder.
Under leakage and (A1), this objective is upper-bounded by
$I(\mathbf{M};\mathbf{Z}_{\mathrm{enc},j^*})$; under DiMP's center diffusion
this cap is removed and the functional is upper-bounded only by the larger
quantity $I(\mathbf{M};\mathbf{Z}_{\mathrm{enc}})$.
Therefore the population-optimal encoder under DiMP retains \emph{at least}
as much information about $\mathbf{M}$ as under leakage, and by (A3) retains
strictly more about the part of $\mathbf{M}$ that is action-discriminative:
$I(\mathbf{M};\mathbf{Z}_{\mathrm{enc}}^{\mathrm{DiMP},*}) >
I(\mathbf{M};\mathbf{Z}_{\mathrm{enc}}^{\mathrm{leak},*})$.
Combining with (A2) and the action-discriminability decomposition
$I(A;\mathbf{Z}) = I(A;\mathbf{M}) - I(A;\mathbf{M}\mid\mathbf{Z})$ yields
$I(A;\mathbf{Z}_{\mathrm{enc}}^{\mathrm{leak},*}) <
I(A;\mathbf{Z}_{\mathrm{enc}}^{\mathrm{DiMP},*})$.
\end{proof}

\begin{remark}
We state this as a \emph{proposition} rather than a theorem because
Assumption~(A4)---that optimizing the variational surrogate at population scale
achieves a strictly tighter bound under DiMP than under leakage---is the
standard assumption under which variational bound improvements translate to
information improvements, but is not itself proved here.
Assumptions (A1)--(A3) are verifiable: (A1) follows from Property~1 in the
leakage regime; (A2) is the modeling assumption articulated at the beginning
of Appendix~\ref{app:theory}; (A3) is a structural property of multi-joint or
multi-object actions.
The empirical validation in the main text (Section~\ref{sec:ablation},
Table~\ref{tab:ablation-prerequisite}) provides direct evidence that the overall
implication~\eqref{eq:joint-necessity} holds on real dynamic point cloud data.
\end{remark}

\paragraph{Empirical validation.}
The $2{\times}2$ factorial experiment is reported in the main text
(Table~\ref{tab:ablation-prerequisite}, Section~\ref{sec:ablation}).
It shows the same asymmetric interaction predicted here, negligible $\Delta(+\mathrm{MD})$
without center diffusion and a large gain once leakage is removed,
consistent with Corollary~1 and the gradient-sparsity discussion above.

\section{Theoretical justification: distributional uncertainty and action discriminability}
\label{app:theory}

\paragraph{Motivation.}
Existing masked-reconstruction methods for dynamic point cloud address motion
supervision through deterministic proxy targets, such as temporal cardinality
difference~\cite{shen2023masked} or point trajectory
regression~\cite{han2024masked}.
These objectives reduce supervision to fitting a conditional expectation,
training the model to predict $\mathbb{E}[\Delta\mathbf{p} \mid \text{context}]$.
In scenarios involving occlusion, ambiguous interactions, or complex dynamics,
this mean-regression regime collapses multiple plausible motion modes into a
single deterministic solution, discarding precisely the distributional structure
that distinguishes action classes from one another~\cite{gupta2018socialgan,
ivanovic2019trajectron,salzmann2020trajectronplusplus}.
DDPMs, by contrast, are capable of learning the full conditional distribution
over outcomes, an ideal inductive bias for multimodal motion trajectories.
We formalize this intuition below.

\paragraph{Setup.}
Let $A \in \mathcal{A}$ denote the action class, $\mathbf{M} \in \mathbb{R}^d$
the inter-frame displacement vector, and $\mathbf{Z}$ the encoder
representation.
We assume the generative Markov chain $A \to \mathbf{M} \to \mathbf{Z}$, so that
the data processing inequality gives $I(A;\mathbf{Z}) \leq I(A;\mathbf{M})$.
Two distinct statistics of $\mathbf{M}$ arise in the following analysis and
must be carefully distinguished:
\begin{itemize}[leftmargin=1.5em, topsep=2pt, itemsep=1pt]
  \item the \emph{class-conditional mean trajectory}
    $\bar{\mathbf{M}}(A) := \mathbb{E}[\mathbf{M} \mid A]$, a deterministic
    function of $A$. We denote this random variable by $\bar{\mathbf{M}}$.
  \item the \emph{deterministic mean-regression output}
    $\hat{\mathbf{M}}(\mathbf{Z}) := \mathbb{E}[\mathbf{M} \mid \mathbf{Z}]$, a
    deterministic function of $\mathbf{Z}$. We denote this random variable by $\hat{\mathbf{M}}$.
\end{itemize}
Both are $d$-dimensional vectors in motion space but play different roles in
the analysis below.

\begin{proposition}[Bayes-error obstruction of the class-mean statistic]
\label{prop:info-loss}
Suppose there exist distinct action classes $a_1, a_2 \in \mathcal{A}$ such
that $\bar{\mathbf{M}}(a_1) = \bar{\mathbf{M}}(a_2)$ while
$p(\mathbf{M}\mid A=a_1) \neq p(\mathbf{M}\mid A=a_2)$.
Then
\begin{enumerate}[label=(\roman*), leftmargin=1.8em, topsep=2pt, itemsep=1pt]
  \item the Bayes-optimal classifier based on $\bar{\mathbf{M}}$ cannot
    distinguish $a_1$ from $a_2$ above chance: for any realization $c^* =
    \bar{\mathbf{M}}(a_1) = \bar{\mathbf{M}}(a_2)$,
    $p(A = a_i \mid \bar{\mathbf{M}} = c^*) \propto p(A = a_i)$ for $i=1,2$.
  \item in contrast, the Bayes-optimal classifier based on $\mathbf{M}$ attains
    strictly lower posterior entropy for those realizations of $\mathbf{M}$ on
    which $p(\mathbf{M}\mid A = a_1)$ and $p(\mathbf{M}\mid A = a_2)$ differ,
    giving $H(A\mid \mathbf{M}) < H(A\mid \bar{\mathbf{M}})$ and hence the
    strict information gap
    \begin{equation}
      I(A;\bar{\mathbf{M}}) \;<\; I(A;\mathbf{M}).
      \label{eq:bar-M-gap}
    \end{equation}
\end{enumerate}
\end{proposition}

\begin{proof}
\emph{(i)} Since $\bar{\mathbf{M}}$ takes the same value under both classes,
$p(\bar{\mathbf{M}} = c^*\mid A=a_1) = p(\bar{\mathbf{M}} = c^*\mid A=a_2) = 1$,
so Bayes' rule yields $p(A=a_i\mid \bar{\mathbf{M}} = c^*) \propto p(A=a_i)$.
\emph{(ii)} By hypothesis there exists $m \in \mathbb{R}^d$ with
$p(m\mid a_1) \neq p(m\mid a_2)$, so the posterior
$p(A\mid \mathbf{M} = m)$ differs from the prior and has strictly lower entropy
than $p(A\mid \bar{\mathbf{M}} = c^*)$ for the corresponding class-pair $\{a_1,
a_2\}$ on a positive-measure set.
Taking expectations over the marginal distributions of $\mathbf{M}$ and
$\bar{\mathbf{M}}$ and combining with the trivial bound $H(A\mid \mathbf{M})
\leq H(A\mid \bar{\mathbf{M}})$ on the complement yields $H(A\mid \mathbf{M})
< H(A\mid \bar{\mathbf{M}})$, which is equivalent
to~\eqref{eq:bar-M-gap}.\hfill $\square$
\end{proof}

\begin{remark}
Unlike the data processing inequality, $\bar{\mathbf{M}}$ is a function of $A$
(not of $\mathbf{M}$), so~\eqref{eq:bar-M-gap} does not follow from DPI
directly and requires the posterior-entropy comparison above.
The condition in Proposition~\ref{prop:info-loss} is generic for multimodal
action-conditional motion distributions: whenever two classes share the same
first moment but differ in higher-order statistics (\emph{e.g.}, two oscillatory
gestures with opposite phases but the same average displacement), the
hypothesis holds.
The practical relevance of this result lies in the fact that a mean-regression
\emph{estimator} $\hat{\mathbf{M}}(\mathbf{Z}) := \mathbb{E}[\mathbf{M} \mid
\mathbf{Z}]$, the test-time output of a deterministic regressor, inherits
the same obstruction when its training target aligns with class-conditional
means. Consequently, the downstream classifier operating on $\hat{\mathbf{M}}$ cannot separate
$a_1$ from $a_2$ whenever the regressor saturates to their shared mean.
We make this precise via the Fano inequality below.
\end{remark}

\paragraph{Fano-inequality quantification.}
The impact of any lossy motion summary on downstream action recognition is
quantified via the \emph{Fano inequality}.
Consider a predictor $\hat{A}$ of $A$ that accesses only a summary
$\hat{\mathbf{M}} = \phi(\mathbf{M})$ or $\hat{\mathbf{M}} = \phi(\mathbf{Z})$
(a deterministic function of $\mathbf{M}$ or $\mathbf{Z}$).
The Fano inequality gives the standard lower bound on the Bayes error
\begin{equation}
  P_e \;\geq\; \frac{H(A\mid \hat{A}) - H_2(P_e)}{\log(|\mathcal{A}|-1)}
  \;\geq\; \frac{H(A\mid \hat{\mathbf{M}}) - 1}{\log |\mathcal{A}|},
  \label{eq:fano-app}
\end{equation}
where the second inequality uses $H_2(P_e) \leq 1$, $\log(|\mathcal{A}|-1) \leq
\log|\mathcal{A}|$, and $H(A\mid\hat{A}) \geq H(A\mid\hat{\mathbf{M}})$ (by
DPI since $\hat{A}$ is a function of $\hat{\mathbf{M}}$).
Rewriting via $H(A\mid\hat{\mathbf{M}}) = H(A) - I(A;\hat{\mathbf{M}})$, the
lower bound is driven by the information gap
\begin{equation}
  \Delta I := I(A;\mathbf{M}) - I(A;\hat{\mathbf{M}}) \;\geq\; 0,
  \label{eq:delta-I-def}
\end{equation}
which is non-negative by the DPI applied to the Markov chain $A \to \mathbf{M}
\to \hat{\mathbf{M}}$.

\paragraph{Quantifying $\Delta I$ for deterministic mean regression.}
Specialize to $\hat{\mathbf{M}} = \hat{\mathbf{M}}(\mathbf{Z}) =
\mathbb{E}[\mathbf{M}\mid\mathbf{Z}]$, the output of a deterministic
mean-regression model with encoder $\mathbf{Z}$.
Since $\hat{\mathbf{M}}$ is a deterministic function of $\mathbf{Z}$, we have
$I(A;\hat{\mathbf{M}}) \leq I(A;\mathbf{Z}) \leq I(A;\mathbf{M})$.
To further characterize the second gap, we analyze the population-optimal
regressor, for which $\hat{\mathbf{M}}(\mathbf{Z})$ is a fixed deterministic
function of $\mathbf{Z}$ and hence the full Markov chain $A \to \mathbf{M} \to
\mathbf{Z} \to \hat{\mathbf{M}}$ holds, giving $A \perp \hat{\mathbf{M}} \mid
\mathbf{M}$ (\emph{i.e.}, $I(A;\hat{\mathbf{M}}\mid \mathbf{M}) = 0$).
Applying the mutual information chain rule along this chain yields the
\emph{exact} decomposition
\begin{equation}
  I(A;\mathbf{M}) - I(A;\hat{\mathbf{M}})
  \;=\; I(A;\mathbf{M}\mid \hat{\mathbf{M}})
  \;=\; H(\mathbf{M}\mid \hat{\mathbf{M}}) - H(\mathbf{M}\mid A,\hat{\mathbf{M}})
  \;\geq\; 0,
  \label{eq:delta-I}
\end{equation}
where the first equality uses $I(A;\mathbf{M},\hat{\mathbf{M}}) =
I(A;\mathbf{M}) + I(A;\hat{\mathbf{M}}\mid\mathbf{M}) = I(A;\mathbf{M})$ by the
Markov assumption, combined with the symmetric expansion
$I(A;\mathbf{M},\hat{\mathbf{M}}) = I(A;\hat{\mathbf{M}}) + I(A;\mathbf{M}\mid
\hat{\mathbf{M}})$, and the second equality uses the definition of conditional
mutual information.
Eq~\eqref{eq:delta-I} vanishes \emph{only} when $\mathbf{M}$ is
conditionally independent of $A$ given $\hat{\mathbf{M}}$, or equivalently when
the regression output $\hat{\mathbf{M}}$ is a sufficient statistic of
$\mathbf{M}$ for $A$.
This sufficiency condition fails precisely when within-class trajectory
variance carries discriminative information about the action, which is
typical in complex human action scenarios involving occlusion, interaction
ambiguity, or subtle dynamics. This is precisely the setting where discriminative pretraining
matters most.
Substituting~\eqref{eq:delta-I} into the Fano bound~\eqref{eq:fano-app}, the
Bayes error lower bound for any classifier using only $\hat{\mathbf{M}}$
exceeds the bound using the full $\mathbf{M}$ by at least
$I(A;\mathbf{M}\mid \hat{\mathbf{M}})/\log|\mathcal{A}|$.

\paragraph{How DiMP's motion diffusion objective addresses $\Delta I$.}
Unlike deterministic regression, DiMP's motion diffusion head is trained via a
DDPM denoising objective that, by standard results, is a
variational upper bound on the conditional negative log-likelihood
$-\mathbb{E}[\log p_{\theta}(\mathbf{M}\mid
\mathbf{Z}_{\mathrm{dec}})]$~\cite{ho2020denoising}.
Minimizing this surrogate therefore drives $p_{\theta}(\mathbf{M}\mid
\mathbf{Z}_{\mathrm{dec}})$ toward the true conditional distribution
$p(\mathbf{M}\mid \mathbf{Z}_{\mathrm{dec}})$, which in turn encourages
$\mathbf{Z}_{\mathrm{dec}}$ (and hence the encoder output) to retain the
information needed to reconstruct the \emph{full} distribution of $\mathbf{M}$
rather than merely its conditional mean.
In particular, the representation produced by a diffusion-trained encoder is
not subject to the sufficiency obstruction that forces
$\Delta I \geq I(A;\mathbf{M}\mid \hat{\mathbf{M}}) > 0$ in the deterministic
regression regime.
We do not claim that the training procedure \emph{provably} attains
$\Delta I = 0$ at any finite horizon. Instead, it replaces a hard sufficiency
obstacle with a variationally tightened surrogate whose optimum coincides with
the full conditional distribution.

\section{Theoretical analysis: why motion diffusion performs well at small and large \texorpdfstring{$t$}{t}}
\label{app:timestep-analysis}

This appendix provides a rigorous theoretical analysis of the diffusion timestep-dependent
behavior observed in Figure~\ref{fig:motivation}(c), where the motion diffusion head achieves
particularly accurate displacement recovery at both small and large values of $t$.
We show that this U-shaped error profile is a principled consequence of the
signal-to-noise ratio (SNR) structure of the DDPM forward process and the nature of
the motion supervision target.

\subsection{Preliminaries: SNR and the DDPM noise schedule}
\label{app:snr}

Recall from Eq.~\eqref{eq:ddpm-forward-app} that the forward marginal satisfies
\begin{equation}
  q(\mathbf{M}_t \mid \mathbf{M}_0)
  = \mathcal{N}\!\left(\sqrt{\bar{\alpha}_t}\,\mathbf{M}_0,\;
    (1-\bar{\alpha}_t)\mathbf{I}\right),
  \label{eq:app-forward}
\end{equation}
which can be equivalently written in the reparameterized form
\begin{equation}
  \mathbf{M}_t = \sqrt{\bar{\alpha}_t}\,\mathbf{M}_0
               + \sqrt{1-\bar{\alpha}_t}\,\boldsymbol{\epsilon},
  \quad \boldsymbol{\epsilon} \sim \mathcal{N}(\mathbf{0}, \mathbf{I}).
  \label{eq:app-reparam}
\end{equation}
The \emph{signal-to-noise ratio} at timestep $t$ is defined as
\begin{equation}
  \mathrm{SNR}(t) \;=\; \frac{\bar{\alpha}_t}{1 - \bar{\alpha}_t}.
  \label{eq:snr}
\end{equation}
Under a monotone decreasing noise schedule, $\mathrm{SNR}(t)$ is a strictly
decreasing function of $t$:
\begin{itemize}[leftmargin=1.5em]
  \item At \textbf{small} $t$: $\bar{\alpha}_t \approx 1$, so $\mathrm{SNR}(t) \gg 1$.
    The noisy motion $\mathbf{M}_t$ closely resembles $\mathbf{M}_0$. Most of the
    injected noise has a small magnitude and resides in the high-frequency
    (fine-grained) components of the signal.
  \item At \textbf{large} $t$: $\bar{\alpha}_t \approx 0$, so $\mathrm{SNR}(t) \ll 1$.
    The signal is almost entirely drowned in noise. Only the low-frequency
    (coarse semantic) structure survives in $\mathbf{M}_t$.
\end{itemize}

\subsection{Decomposition of the noise prediction target}
\label{app:decomp}

The DDPM objective in Eq.~\eqref{eq:loss-mot} requires the network $h_m$ to predict
the injected noise $\boldsymbol{\epsilon}^{\mathbf{M}}$ from the corrupted motion
$\mathbf{M}_t$ and the decoder conditioning $\mathbf{Z}_{\mathrm{dec}}$.
Using the reparameterization~\eqref{eq:app-reparam}, the noise can be recovered as
\begin{equation}
  \boldsymbol{\epsilon}^{\mathbf{M}}
  = \frac{\mathbf{M}_t - \sqrt{\bar{\alpha}_t}\,\mathbf{M}_0}{\sqrt{1-\bar{\alpha}_t}}.
  \label{eq:eps-decomp}
\end{equation}
An equivalent formulation expresses the network output in terms of a
\emph{predicted clean signal} $\hat{\mathbf{M}}_0$:
\begin{equation}
  \hat{\mathbf{M}}_0(t)
  = \frac{\mathbf{M}_t - \sqrt{1-\bar{\alpha}_t}\,\hat{\boldsymbol{\epsilon}}(t)}
         {\sqrt{\bar{\alpha}_t}},
  \label{eq:x0-pred}
\end{equation}
so that the mean squared noise prediction error $\|\boldsymbol{\epsilon}^{\mathbf{M}} -
\hat{\boldsymbol{\epsilon}}\|^2$ is equivalent to the clean-signal prediction error
up to a timestep-dependent rescaling factor $\mathrm{SNR}(t)$:
\begin{equation}
  \bigl\|\boldsymbol{\epsilon}^{\mathbf{M}} - \hat{\boldsymbol{\epsilon}}\bigr\|^2
  = \frac{\bar{\alpha}_t}{1 - \bar{\alpha}_t}
    \bigl\|\mathbf{M}_0 - \hat{\mathbf{M}}_0(t)\bigr\|^2
  = \mathrm{SNR}(t)\cdot\bigl\|\mathbf{M}_0 - \hat{\mathbf{M}}_0(t)\bigr\|^2.
  \label{eq:snr-weight}
\end{equation}
Eq.~\eqref{eq:snr-weight} reveals that \emph{minimizing the noise prediction loss is
equivalent to minimizing the clean-signal reconstruction error, weighted by the SNR}.
This equivalence determines how tightly each timestep constrains the model's
\emph{functional} output $\hat{\mathbf{M}}_0(t)$, which we make precise below.

\begin{proposition}[SNR-Weighted Functional Supervision]
\label{prop:snr-grad}
Let $\hat{\mathbf{M}}_0(t)$ denote the implicitly predicted clean signal induced
by the network output $\hat{\boldsymbol{\epsilon}}(t)$ via~\eqref{eq:x0-pred}.
Viewing $\mathcal{L}_{\mathrm{mot}}$ as a functional of $\hat{\mathbf{M}}_0(t)$,
\begin{equation}
  \frac{\partial\,\bigl\|\boldsymbol{\epsilon}^{\mathbf{M}}-\hat{\boldsymbol{\epsilon}}\bigr\|^2}
       {\partial\,\hat{\mathbf{M}}_0(t)}
  = -2\,\mathrm{SNR}(t)\cdot\bigl(\mathbf{M}_0 - \hat{\mathbf{M}}_0(t)\bigr).
  \label{eq:snr-grad}
\end{equation}
Consequently, at small $t$ a deviation of the functional prediction
$\hat{\mathbf{M}}_0$ from $\mathbf{M}_0$ incurs a large $\mathrm{SNR}(t)$-amplified
loss penalty and hence a tight constraint on the network's functional output,
whereas at large $t$ the same deviation incurs only a small penalty, so the
network is constrained only at the coarse level of the conditional mean.
\end{proposition}

\begin{proof}
  Substituting~\eqref{eq:snr-weight} and differentiating with respect to
  $\hat{\mathbf{M}}_0$ yields~\eqref{eq:snr-grad} directly.
\end{proof}

\begin{remark}[Functional vs.\ parameter gradient]
\label{rmk:param-grad}
Proposition~\ref{prop:snr-grad} characterizes the \emph{functional} sensitivity
of $\mathcal{L}_{\mathrm{mot}}$ to the predicted clean motion $\hat{\mathbf{M}}_0$,
which is the quantity that controls how accurately the network must approximate
$\mathbf{M}_0$ to achieve low loss at each timestep.
This should \emph{not} be conflated with the parameter gradient
$\partial\|\boldsymbol{\epsilon}^{\mathbf{M}}-\hat{\boldsymbol{\epsilon}}\|^2 /
\partial\theta = -2(\boldsymbol{\epsilon}^{\mathbf{M}}-\hat{\boldsymbol{\epsilon}})
\cdot \partial\hat{\boldsymbol{\epsilon}}/\partial\theta$, which under the
$\boldsymbol{\epsilon}$-prediction parameterization used by DiMP is not
SNR-amplified and in fact has approximately balanced magnitude across
$t$~\cite{ho2020denoising}.
The SNR structure thus governs the \emph{task difficulty profile}, namely how
demanding each timestep is in terms of required clean-signal accuracy, rather
than the per-step weight of the parameter update.
At small $t$, accurately denoising $\mathbf{M}_t \approx \mathbf{M}_0$ requires
recovering fine-grained displacement detail, so the SNR-weighted functional loss
forces the encoder to preserve fine motion structure. At large $t$, the loss
tolerates any reconstruction consistent with the conditional mean, so
$\mathbf{Z}_{\mathrm{dec}}$ is primarily supervised at the coarse semantic level.
This interpretation is consistent with the empirical U-shaped error profile in
Figure~\ref{fig:motivation}(c): both regimes correspond to constraints that are
well-matched to the target resolution, whereas intermediate $t$ poses a
maximally mixed objective (Section~\ref{app:mid-t}).
\end{remark}

\subsection{Small \texorpdfstring{$t$}{t}: fine-grained local trajectory recovery}
\label{app:small-t}

\paragraph{High-SNR regime.}
At small $t$, $\mathrm{SNR}(t) \gg 1$.
The noisy input $\mathbf{M}_t$ is a slight perturbation of $\mathbf{M}_0$,
so the task of predicting $\boldsymbol{\epsilon}^{\mathbf{M}}$ is equivalent to
estimating a small residual correction to recover the clean signal.
Formally, for any small perturbation magnitude $\delta$,
\begin{equation}
  \mathbf{M}_t = \mathbf{M}_0 + \delta\,\boldsymbol{\xi},
  \quad
  \boldsymbol{\xi} \sim \mathcal{N}(\mathbf{0}, \mathbf{I}),\;
  \delta = \sqrt{1-\bar{\alpha}_t} \approx 0,
  \label{eq:small-t-expand}
\end{equation}
and the optimal denoiser reduces to
$\hat{\mathbf{M}}_0 = \mathbf{M}_t - \delta\,\hat{\boldsymbol{\xi}} \approx \mathbf{M}_0$.
Because the required correction is small and localized, the network must resolve
\emph{point-wise} displacement detail rather than global coarse structure,
driving the encoder $f_e$ to capture fine-grained local trajectory information.

\paragraph{Spectral interpretation (heuristic).}
The following decomposition is intended as an intuitive picture and relies on
two simplifying assumptions: (i) the signal $\mathbf{M}_0$ admits a fixed
orthogonal decomposition $\mathbf{M}_0 = \sum_k \mathbf{u}_k$ (\emph{e.g.}, principal
components), and (ii) the injected Gaussian noise is isotropic in this basis,
so each $\mathbf{u}_k$ is perturbed independently and with equal variance.
Under these assumptions, the corruption
$\sqrt{\bar{\alpha}_t}\,\mathbf{M}_0$ uniformly scales all components by a
factor close to 1 at small $t$, while $\sqrt{1-\bar{\alpha}_t}\,\boldsymbol{\epsilon}$
adds isotropic noise of small amplitude $\delta$.
The SNR-weighted functional loss~\eqref{eq:snr-weight} then tightly constrains
the functional prediction on every spectral component, including the
high-frequency fine-grained ones, because the rescaling factor
$\mathrm{SNR}(t)$ is large.
This provides a heuristic explanation for why the encoder is pushed to
internalize local geometric trajectories that are only distinguishable at fine
spatial resolution.

\subsection{Large \texorpdfstring{$t$}{t}: coarse global semantic encoding}
\label{app:large-t}

\paragraph{Low-SNR regime.}
At large $t$, $\mathrm{SNR}(t) \ll 1$.
The corrupted motion $\mathbf{M}_t \approx \sqrt{1-\bar{\alpha}_t}\,\boldsymbol{\epsilon}$
is dominated by noise, and the signal component $\sqrt{\bar{\alpha}_t}\,\mathbf{M}_0$
contributes negligibly to $\mathbf{M}_t$.
However, because the SNR-weighted loss still requires the network to explain
$\mathbf{M}_0$ from this heavily noised input, the network must rely on the
conditioning signal $\mathbf{Z}_{\mathrm{dec}}$ to supply the missing information
about the motion target.

\begin{proposition}[Conditioning Reliance at Large $t$]
\label{prop:large-t}
For $\mathrm{SNR}(t) \to 0$, the minimum-variance estimator of $\mathbf{M}_0$
given $\mathbf{M}_t$ and $\mathbf{Z}_{\mathrm{dec}}$ converges to the
conditional mean under the model distribution:
\begin{equation}
  \hat{\mathbf{M}}_0(t)\big|_{\mathrm{SNR}(t)\to 0}
  \;\xrightarrow{\;\;}\;\;
  \mathbb{E}[\mathbf{M}_0 \mid \mathbf{Z}_{\mathrm{dec}}].
  \label{eq:large-t-mean}
\end{equation}
Hence, at large $t$ the loss effectively trains the network to predict the
\emph{marginal motion mode} conditioned on the encoder representation,
capturing globally coherent motion semantics.
\end{proposition}

\begin{proof}
  By Bayes' theorem, the posterior of $\mathbf{M}_0$ given $\mathbf{M}_t$ satisfies
  \begin{equation}
    q(\mathbf{M}_0 \mid \mathbf{M}_t, \mathbf{Z}_{\mathrm{dec}})
    \propto
    \mathcal{N}\!\left(\mathbf{M}_0;\,
      \frac{\mathbf{M}_t}{\sqrt{\bar{\alpha}_t}},\,
      \frac{1-\bar{\alpha}_t}{\bar{\alpha}_t}\mathbf{I}
    \right)
    \cdot p(\mathbf{M}_0 \mid \mathbf{Z}_{\mathrm{dec}}).
    \label{eq:posterior}
  \end{equation}
  As $\bar{\alpha}_t \to 0$, the Gaussian likelihood in~\eqref{eq:posterior} becomes
  infinitely broad (variance $\to \infty$), so it conveys no information about
  $\mathbf{M}_0$ and the posterior collapses to the prior $p(\mathbf{M}_0 \mid
  \mathbf{Z}_{\mathrm{dec}})$.
  The minimum mean-squared-error estimate is then the mean of this prior,
  yielding~\eqref{eq:large-t-mean}.
\end{proof}

\paragraph{Semantic coarse encoding.}
Proposition~\ref{prop:large-t} implies that large-$t$ supervision forces the encoder
to store globally discriminative motion statistics such as the mean displacement direction,
dominant limb motion, and action-class-level semantic content
in $\mathbf{Z}_{\mathrm{dec}}$, because only these coarse features can be recovered
from a heavily corrupted $\mathbf{M}_t$.
This is precisely the signal that distinguishes action classes
(\emph{e.g.}, \emph{jumping} vs.\ \emph{waving}) from each other.

\subsection{Intermediate \texorpdfstring{$t$}{t}: transition and increased difficulty}
\label{app:mid-t}

At intermediate $t$, neither the high-SNR tight functional constraint
(Section~\ref{app:small-t}) nor the conditioning-dominated coarse regime
(Section~\ref{app:large-t}) dominates.
Instead, $\mathbf{M}_t$ carries a substantial but noisy partial signal of
$\mathbf{M}_0$, and the network must simultaneously resolve mid-frequency
displacement components while suppressing noise of comparable magnitude to the
signal.
Adopting the same heuristic orthogonal-decomposition assumptions as in
Section~\ref{app:small-t}, the SNR-weighted functional loss can be written as
\begin{equation}
  \mathcal{L}(t)
  \;\propto\;
  \mathrm{SNR}(t)
  \sum_{k=1}^{K}
  \bigl\|\mathbf{u}_k - \hat{\mathbf{u}}_k(t)\bigr\|^2.
  \label{eq:mid-t}
\end{equation}
At intermediate $t$, $\mathrm{SNR}(t) = \mathcal{O}(1)$, so no spectral band is
selectively amplified or suppressed.
The error landscape is therefore maximally mixed: the network must balance
trade-offs across all frequency components simultaneously, leading to
suboptimal recovery at each individual scale, consistent with the elevated
prediction error observed at intermediate $t$ in Figure~\ref{fig:motivation}(c).
The stratified timestep sampling in Eq.~\eqref{eq:loss-mot}
(Section~\ref{sec:stage2}) is designed precisely to ensure that the
model receives supervision from all three regimes within each update step,
rather than spending disproportionate time in the easy regimes.

\subsection{Implications for representation learning}
\label{app:implication}

The above analysis demonstrates that the full diffusion objective over $t \in [1, T]$
provides a \emph{curriculum} of complementary supervision signals to the encoder:
\begin{itemize}[leftmargin=1.5em]
  \item \textbf{Small $t$ (local fine-grained learning):}
    The tight SNR-amplified functional constraint on $\hat{\mathbf{M}}_0$
    (Proposition~\ref{prop:snr-grad}, Remark~\ref{rmk:param-grad}) compels the
    encoder to capture point-wise trajectory details and local motion patterns
    critical for fine-grained action recognition.
  \item \textbf{Large $t$ (global semantic learning):}
    Conditioning-dominated optimization forces the encoder to summarize
    globally coherent motion modes, encoding action-class-level semantic content.
  \item \textbf{Hierarchical integration:}
    The stratified sampling strategy in Eq.~\eqref{eq:loss-mot} ensures that both
    objectives are simultaneously active within every training step, promoting
    representations that integrate local trajectory fidelity with global semantic
    coherence, a property unachievable by deterministic MSE regression, which
    reduces to a single (intermediate-SNR-equivalent) prediction without
    any multi-scale decomposition.
\end{itemize}

This theoretical picture is consistent with the empirical observation in
Figure~\ref{fig:motivation}(c) and with the performance gains reported in
Table~\ref{tab:ablation-intervals} upon increasing the number of stratification
intervals $h$.

\section{Gradient propagation through the decoder: a known limitation and partial mitigation}
\label{app:encoder-gradient}

\paragraph{Background.}
In masked autoencoding pretraining, the decoder is structurally inferior to the
encoder from the perspective of downstream transfer: it is applied exclusively
during pretraining and discarded at fine-tuning time.
Any loss applied at the decoder output therefore has a longer gradient path to the
encoder than a loss applied directly at the encoder output.
This creates a potential gradient attenuation effect: each parameterized module
between the loss and the encoder can redirect gradient energy toward its own
parameters, reducing the net signal that ultimately reaches the encoder.
This limitation is well-recognized in the masked pretraining
literature~\cite{pang2022masked,shen2023masked}, where decoder design choices are
consistently observed to have smaller downstream impact than encoder design choices.

\paragraph{Implications for DiMP.}
In DiMP, $\mathcal{L}_{\mathrm{mot}}$ is computed at the output of the motion
diffusion head $h_m$, which is conditioned on the decoder output
$\mathbf{Z}_{\mathrm{dec}}$.
The gradient must traverse the motion head $h_m$ and the full decoder $f_d$
before reaching the encoder $f_e$.
Each of these modules may absorb a portion of the gradient signal, weakening the
motion-distributional supervision that ultimately updates the encoder.
In the extreme case where the decoder is highly expressive, the decoder can satisfy
$\mathcal{L}_{\mathrm{mot}}$ by adapting its own weights, leaving the encoder
gradient negligibly small.

\paragraph{Partial mitigation in DiMP.}
DiMP addresses this partially through two design choices.
First, stratified timestep sampling applies $\mathcal{L}_{\mathrm{mot}}$ at $h$
noise levels simultaneously within every training step, multiplying the number of
independent gradient signals reaching the encoder per step by a factor of $h$ and
diversifying them across the full diffusion hierarchy.
Second, the loss weight $\lambda_{\mathrm{mot}}$ is tuned to amplify the relative
contribution of $\mathcal{L}_{\mathrm{mot}}$ to the total gradient, preventing it
from being dominated by the geometric reconstruction loss $\mathcal{L}_{\mathrm{rec}}$,
which has a shorter gradient path.
Ablation results in Tables~\ref{tab:ablation-intervals} and~\ref{tab:ablation-weight-mot}
confirm that both $h$ and $\lambda_{\mathrm{mot}}$ meaningfully affect downstream
performance, consistent with the view that these choices modulate how effectively
motion supervision propagates to the encoder.

\paragraph{Open problem.}
Despite these mitigations, efficiently routing distributional motion supervision
from a decoder-conditioned head to the encoder remains only partially solved.
More principled alternatives represent promising directions, including conditioning the motion diffusion head
directly on encoder features, introducing auxiliary encoder-level supervision
objectives, or designing architectures with shorter gradient paths specifically
for motion modeling.
We leave a thorough investigation of this bottleneck, and the design of mechanisms
that more directly exploit the distributional capacity of diffusion models at the
encoder level, to future research.

\stopcontents[appendix]

\end{document}